\documentclass[journal, twocolumn]{IEEEtran}

\ifCLASSINFOpdf
   \usepackage[pdftex]{graphicx}
   \graphicspath{{../pdf/}{../jpeg/}}
   \DeclareGraphicsExtensions{.pdf,.jpeg,.png}
\else
   \usepackage[dvips]{graphicx}
   \graphicspath{{../eps/}}
\fi

\ifCLASSOPTIONcompsoc
  \usepackage[caption=false,font=normalsize,labelfont=sf,textfont=sf]{subfig}
\else
  \usepackage[caption=false,font=footnotesize]{subfig}
\fi

\usepackage[cmex10]{amsmath}
\usepackage{bm}
\usepackage{amssymb}
\usepackage{amsthm}

\usepackage[ruled,linesnumbered]{algorithm2e}
\usepackage{algpseudocode}
\usepackage{array}
\usepackage{cite}
\usepackage{url}
\usepackage{booktabs}
\usepackage{ragged2e}
\renewcommand{\raggedright}{\leftskip=0pt \rightskip=0pt plus 0cm}

\renewcommand{\raggedright}{\leftskip=0pt \rightskip=0pt plus 0cm}

\newcommand{\etal}{\emph{et~al.~}}

\newtheorem{defi}{Definition}[section]
\newtheorem{lemm}{Lemma}[section]
\newtheorem{theo}{Theorem}[section]

\makeatletter

\newcommand{\Rmnum}[1]{\expandafter\@slowromancap\romannumeral #1@}
\makeatother

\begin{document}

\title{An Online Ride-Sharing Path Planning Strategy for Public Vehicle Systems}

\author{Ming~Zhu,
        ~Xiao-Yang~Liu,
        ~and~Xiaodong~Wang,~\IEEEmembership{Fellow,~IEEE}
\thanks{M.~Zhu is with the Shenzhen Institutes of Advanced Technology, Chinese Academy of Sciences, Shenzhen, China, and the Department of Computer Science and Engineering, Shanghai Jiao Tong University, Shanghai, China, E-mail: zhumingpassional@gmail.com, zhumingpassional@sjtu.edu.cn.} 
\thanks{X.-Y.~Liu is with the Electrical Engineering Department, Columbia University, New York City, US, and the Department of Computer Science and Engineering, Shanghai Jiao Tong University, Shanghai, China. E-mail: xiaoyang@ee.columbia.edu.}
\thanks{X.~Wang is with the Department of Electrical Engineering, Columbia University, New York, NY, 10027, USA (E-mail: wangx@ee.columbia.edu).}

}
%

\maketitle
%
%
%
%

\begin{abstract}
  As efficient traffic-management platforms, public vehicle (PV) systems are envisioned to be a promising approach to solving traffic congestions and pollutions for future smart cities. PV systems provide online/dynamic peer-to-peer ride-sharing services with the goal of serving sufficient number of customers with minimum number of vehicles and lowest possible cost. A key component of the PV system is the online ride-sharing scheduling strategy. In this paper, we propose an efficient path planning strategy that focuses on a limited potential search area for each vehicle by filtering out the requests that violate passenger service quality level, so that the global search is reduced to local search. We analyze the performance of the proposed solution such as reduction ratio of computational complexity. Simulations based on the Manhattan taxi data set show that, the computing time is reduced by 22\% compared with the exhaustive search method under the same service quality performance.
\end{abstract}


\IEEEpeerreviewmaketitle

\section{Introduction}\label{Sec:Introduction}

Existing transportation systems are not satisfying due to large investments, high congestions, serious pollutions and potential health problems \cite{litman2013transportationHealth}. We hope that our transportation policies and programs can serve our economic, social, and environmental goals. Some research shows that the total industrial and consumer expenditure on transportation is about 10\% of GDP (gross domestic product) in the world \cite{pisarski2008transportation}. In the USA, around 16.7\% of the household income is spent on transportation \cite{CommutingAmerica}. Besides the large investments, many countries are suffering from social problems such as traffic congestions \cite{zhu2013lane-change} and fuel pollution-related diseases \cite{hystad2013AirPollutionLungCancer} brought by transportation. The vehicle fuel pollution accounts for  31\% of all pollution in Beijing \cite{zhu2015TrafficEfficiency}, and the Beijing city government has spent \$277 billion during 2011-2013 \cite{zhu2015TrafficEfficiency} on air pollution. However, the PM2.5 (particulate matter with diameter of 2.5 micrometers or less) level is still high \cite{BeijingAirQualityIndex}. To address the above problems, ride-sharing is a promising solution that is cost-effective.

As an application of sharing economy \cite{hamari2016sharing-economy}, public vehicle (PV) systems \cite{zhu2016PublicVehicle} \cite{zhu2017JointTransportationCharging} \cite{zhu2016PSA-IWQoS} \cite{zhu2015PublicVehicle} provide low-cost peer-to-peer ride-sharing trips with ensured quality of service (QoS) for passengers. As shown in Fig.~\ref{Fig:PVArchitecture}, a PV system consists of three parts: a cloud, PVs, and passengers. The blue solid lines denote the communications between them, and the dash black lines imply the scheduling task of a PV. In addition, future PVs are envisioned to be self-driving electric vehicles, therefore charging \cite{zhu2014ChargingScheduling} is an important issue.

\begin{figure}[tbp]
  \centering
  \includegraphics[height=0.7\linewidth,width=0.87\linewidth]{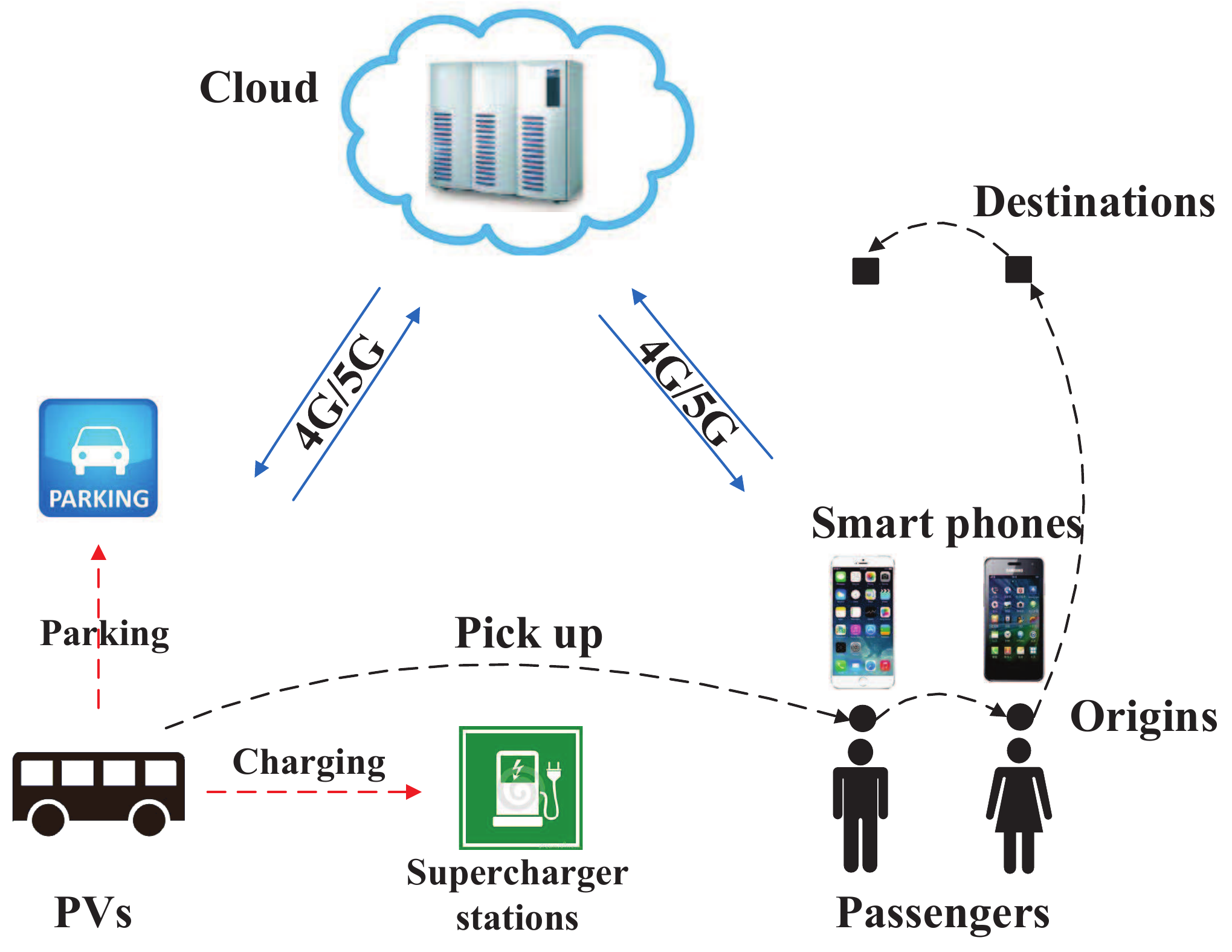}
  \caption{Architecture of PV systems.}
  \label{Fig:PVArchitecture}
\end{figure}

The operation of a PV system is as follows. If a passenger needs a trip service, he/she sends a request (including the earliest start time, the origin, and the destination, etc.) through a mobile internet device (e.g., smart phone) to a cloud. The cloud computes a path plan with a confirmed ride match to serve him/her. He/She can access the information of the confirmed PV through apps, e.g., vehicle ID, position, speed, and path. PVs serve passengers by traversing from origins (pick-up points) to destinations (drop-off points). Compared with conventional public transportation systems such as buses and subways which also provide ride-sharing services, there is no last mile problem in PV systems, and PVs are more flexible since their paths are adapted to the trip demands of passengers, while the paths and the schedules of buses and subways are fixed.

PV systems are also different from existing ride-sharing systems (e.g., Uber Pool \cite{Uber} and Didi \cite{Didi}). Here, we describe several important differences. 1) PV systems are centralized systems which provide online/dynamic peer-to-peer ride-sharing services, while existing ride-sharing systems such as Uber Pool are distributed offline/static systems. 2) In PV systems, scheduling strategies are calculated by the cloud, while in existing ride-sharing systems, scheduling strategies are negotiated by drivers and riders. 3) PVs cooperate with each other to improve traffic efficiency, e.g., multi-hop ride-sharing paths \cite{zhu2016transfer}, while in existing ride-sharing systems, drivers compete for more profit.

In this paper, we investigate the efficient ride-sharing path planning problem in PV systems. The existing solutions \cite{zhu2016PublicVehicle} \cite{shang2016collective-travel-planning} \cite{jung2016Shared-Taxi} become inefficient since most of them are based on exhaustive search, and only some of them \cite{ma2015taxiRidesharing} considers computational efficiency while the QoS can not be guaranteed. We study this problem in a practical setting by exploiting QoS constraints and geometry, and then propose a local search solution.

The challenges of our problem are as follows. 1) There is a trade-off between the objectives of passengers and PVs: PVs try to serve more passengers with the minimum energy cost or travel distance; where passengers want to arrive at their destinations as early as possible with low cost. 2) The online/dynamic ride-sharing in PV systems involves multiple passengers' utilities. If the utility of a passenger is compromised by serving some other passengers, e.g., a long detour, he/she would choose other type of service such as taxi. 3) When the cloud receives new trip requests, new schedules will be calculated and paths of some PVs may change.

The contributions of this paper include the following.
\begin{itemize}
\item{}
  To the best of our knowledge, this is the first work that considers computational efficiency in online/dynamic ride-sharing for PV systems aiming at reducing the travel distance of vehicles with QoS guarantee.
\end{itemize}

\begin{itemize}
\item{}
  We propose an efficient path planning strategy by restricting the search areas for PVs, reducing the global search to local search. It is suitable for real-time implementation.
\end{itemize}

\begin{itemize}
\item{}
  We analyze the performance of the proposed solution, and find that the larger is the city, the better performance (reduction ratio of computational complexity) our proposed solution will have.
\end{itemize}

\begin{itemize}
\item{}
  The strategy balances the utilities of passengers and PVs, i.e., providing high QoS (e.g., short waiting time, and less detour) for passengers with low energy cost.
\end{itemize}

\begin{itemize}
\item{}
  We perform simulations with the trip requests based on the Manhattan taxi data set to evaluate the proposed strategy. A large amount of computation can be saved. The travel distance of PVs is reduced, and the ride-sharing for passengers can be guaranteed compared with privately owned electric vehicles.
\end{itemize}

Related work on ride-sharing path planning problems in PV or PV-like systems is described in Section \ref{Sec:RelatedWork}. In Section \ref{Sec:PathPlanningProblem}, we present the problem formulation. Section \ref{Sec:Approach} and Section \ref{Sec:Analysis} detail the solution and its performance analysis respectively. Section \ref{Sec:Performance} provides simulation results. Section \ref{Sec:Conclusion} concludes this paper.

\section{Related Work} \label{Sec:RelatedWork}


Some ride-sharing path planning strategies are restricted to some special cases such as common origins or destinations. In particular, Massobrio \etal \cite{massobrio2016taxi-sharing} propose evolutionary algorithms to solve the one-origin-multi-destination taxi-sharing problem, where the QoS metrics include the total trip cost and time delay of passengers. Naoum-Sawaya \etal \cite{naoum2015stochastic} present a stochastic mixed integer programming model to optimize the allocation of cars to employees (from homes to work places) while taking into account the unforeseen events of car unavailability. It only focuses on ride-sharing with common destinations in large organizations, e.g., companies, hospitals, and universities. Shang \etal \cite{shang2016collective-travel-planning} propose a collective travel planning query to find the lowest cost path connecting multiple origins and a destination with limited number of meeting points. However, all the above solutions can not be used in the multi-origin-multi-destination scenario, which is more common in the real world.

Most of current ride-sharing path planning strategies in the multi-origin-multi-destination scenario are exhaustive search methods, which incur high computational load at the cloud. Zhu \etal \cite{zhu2016PublicVehicle} propose an algorithm to reduce travel distance of vehicles with QoS constraints for passengers, e.g., detour, which needs to try each request and calculate corresponding detour constraints. Goel \etal \cite{goel2017ridesharing} propose a solution which selects the optimal fixed positions of pick-up points to maximize the vehicle occupancy rates while preserving the passenger privacy and safety, where passengers do not need to provide their precise home or work positions. To maximize the vehicle occupancy and minimize the travel time with limited detour, Jung \etal \cite{jung2016Shared-Taxi} propose a hybrid-simulated annealing to dynamically assign passenger requests for online/dynamic ride-sharing, which is computationally expensive especially when the number of requests is large since multiple random perturbations and a large number of iterations are needed.

Less work focuses on how to reduce the computational complexity and restrict the search areas in online/dynamic ride-sharing in PV or PV-like systems \cite{ma2015taxiRidesharing}. To reduce the total travel distance of taxis, Ma \etal \cite{ma2015taxiRidesharing} propose an efficient ride-sharing path planning solution in serving dynamic queries, where ``lazy shortest path" calculation is used by means of partitioning the whole road network into multiple grids, and the status of each vehicle should be updated according to preset intervals. This work only considers how to reduce the computational complexity based on the current distance between vehicles and origins/destinations of requests. Ota \etal \cite{ota2016TaxiRideSharing} propose a real-time and data-driven simulation framework to achieve the efficient analysis of taxi ride sharing by means of exploring parallelism and cache-coherent shortest path index, which also considers different stakeholders' interests and constraints, e.g., the waiting time, the maximum number of additional stops, the maximum number of shared trips.

In summary, most online/dynamic ride-sharing solutions in PV or PV-like systems are based on computationally expensive exhaustive search. Only a few solutions focus on reducing the computational complexity \cite{ma2015taxiRidesharing} \cite{ota2016TaxiRideSharing}. However, the road network should be divided into cells and the states of vehicles should be frequently updated, making the methods complicated. Our proposed solution in this paper is more efficient, easy to implement, and can be used in large cities. Moreover, the QoS level of passengers can be guaranteed, e.g., short waiting time and less detour.

\section{PV Ride-Sharing Planning Problem} \label{Sec:PathPlanningProblem}

In this section, we first present preliminaries for our problem, and then detail the problem formulation and NP-completeness.

\subsection{Preliminaries}

We assume that all PVs in a city constitute a set $\mathcal{P}$, and all requests constitute a set $\mathcal{R}$. Let $\mathcal{R}_u$ be a set of unscheduled requests, and $\mathcal{R}_s$ be a set of scheduled requests. We assume that once the ride match between any request $r$ and PVs has been confirmed by the cloud, it does not change. Let $\mathcal{R}_s = \mathcal{R}_{s,1} \bigcup \mathcal{R}_{s,2}$, where $\mathcal{R}_{s,1}$ is a set of requests being served (has been picked up yet not dropped off) by PVs and $\mathcal{R}_{s,2}$ is the set of requests waiting to be served (the ride match has been confirmed yet not picked up). Clearly, $\mathcal{R} = \mathcal{R}_u \bigcup \mathcal{R}_s$.

In the PV ride-sharing planning (PVRSP) problem, each request $r \in \mathcal{R}$ should be served by a corresponding PV, while on the path of any PV, it can serve other unscheduled requests if there exist available seats. Let $r = (n, t, o, d) \in \mathcal{R}$ denote a trip request, where $n$ is the number of passengers, $t$ is the earliest start time, $o$ is the origin (pick-up point), and $d$ is the destination (drop-off point). We assume the passengers using the same request should be served together by a PV. For request $r_1 \in \mathcal{R}_{s,1}$, it has been picked up so that its origin is not important anymore and only its destination should be reached. However, for request $r_2 \in \mathcal{R}_{s,2}$, both the origin and the destination should be reached with the origin preceding the destination. Let $p \in \mathcal{P}$ denote a PV or its current position.

Next, to describe our problem clearly, we introduce three definitions: schedule, service list, and path. And then we present an example to discuss their changes in the execution of a path planning strategy in PV systems.

\begin{defi} \label{definition:Schedule}
The schedule for any request $r = (n, t, o, d)$, means that if a PV $p$ is scheduled to $r$ with the earliest start time $t$, $p$ will transverse through the origin $o$ and the destination $d$ with $o$ preceding $d$, and meanwhile $p$ will pick up $r$ ($n$ passengers) at $o$, and drop off $r$ at $d$.
\end{defi}

In Definition~\ref{definition:Schedule}, a schedule mainly determines the ride match between the request and PVs, and the precedence constraints between the origin and the destination. The schedule for a request makes sure that this request will be transported from the origin to the destination.

\begin{defi} \label{definition:ServiceList}
The service list of any PV $p$, denoted by $L_p$, is a list of requests $p$ has to serve, including the requests being served (have been picked up by $p$, yet have not arrived at their destinations), and the requests waiting to be served (scheduled to $p$, yet have not been picked up), while the requests dropped off by $p$ are not included.
\end{defi}

In Definition \ref{definition:ServiceList}, the service list $L_p$ of PV $p$ clearly points out which requests should be served and implies where to pick up or drop off corresponding passengers. Let $\mathcal{L}$ be a set of service lists of all PVs. $L_p$ has two parts, the requests being served, and the requests waiting to be served by $p$. Once a new request is assigned to PV $p$, it will be put to the service list. Once a scheduled request is dropped off at its destination, it will be removed from the service list. We see that, the service list is determined by both the ride-sharing planning strategy and the motion states of PVs, and it dynamically changes over time.

\begin{defi} \label{definition:Path}
The path of PV $p$, denoted by $Q_p$, is a sequence of points $p$ has to serve, which includes the current position of this PV, the destinations of the requests being served by it and the origin-destination pairs of the requests waiting to be served by it.
\end{defi}

In Definition \ref{definition:Path}, we see that the path of each PV dynamically changes over time, and is composed of the current position of this PV, the origin-destination pairs of requests waiting to be served and the destinations of requests being served. Let $\mathcal{Q}$ be a set of paths of all PVs. All the important variables and notations in this article are summarized in Table~\ref{Tab:Variables1}, where PSA is the abbreviation of ``potential search area" which will be detailed in the subsequent sections, and the unit of $D(i,j)$, $E(i,j)$, $T_r$, $b$, and $B$ is km.

\begin{table}[tbp]
  \centering
  \caption{Variables and notations}\label{Tab:Variables1}
  \begin{tabular}{ l  l}
  \toprule
  $n$ & number of passengers of $r$. \\
  $t$ & earliest start time of $r$. \\
  $o$ & origin of $r$. \\
  $d$ & destination of $r$. \\
  $\delta$ & detour ratio of request $r$.\\
  $\Delta$ & maximum detour ratio.\\
  $b$ & buffer distance of request $r$.\\
  $B$ & buffer distance threshold.\\
  $w$ & waiting time of request $r$.\\
  $W$ & waiting time threshold.\\
  $C$ & capacity of PVs.\\
  $L_p$ & service list of PV $p$.\\
  $Q_p$ & path of PV $p$.\\
  $D(i,j)$ & shortest path distance from position $i$ to $j$.\\
  $E(i,j)$ & Euclidean distance from position $i$ to $j$.\\
  $T_r$ & travel distance of request $r$.\\
  $\mathcal{P}$ & a set of all PVs in a city. \\
  $p$ & a PV.\\
  $p_s$ & position of PV $p$ at the schedule time of $r$.\\
  $r$ & a request. \\
  $\mathcal{R}$ & a set of all requests.  \\
  $\mathcal{R}_s$ & a set of scheduled requests. \\
  $\mathcal{R}_{s,1}$  &  a set of requests being served.\\
  $\mathcal{R}_{s,2}$ & a set of requests waiting to be served.\\
  $L_p$ & service list of PV $p$.\\
  $\mathcal{L}$ & a set of service lists of all PVs.\\
  $Q_p$ & path of PV $p$.\\
  $\mathcal{Q}$ & a set of paths of all PVs.\\
  $\eta$ & ratio of PSA and optimal PSA ($\text{PSA}^{\text{opt}}$). \\
  $\hat{A}_1$/$A_1$ & PSA of $p$ if the furthest request $r$ has been picked up,\\
   & and $A_1 = |\hat{A}_1|$. \\
  $\hat{A}_2$/$A_2$ & PSA of $p$ if the furthest request $r$ has not been picked\\
   & up, and $A_2 = |\hat{A}_2|$. \\
  $\hat{A}$/$A$ & PSA of PV $p$, and $A = |\hat{A}|$.\\
  $\hat{A}^{\text{opt}}$/$A^{\text{opt}}$ & $\text{PSA}^{\text{opt}}$ of $p$, and $A^{\text{opt}} = |\hat{A^{\text{opt}}}|$. \\
  $\hat{\beta}$/$\beta$ & PSA determined by $(o, d)$, and $\beta = |\hat{\beta}|$.\\
  $\hat{\beta}^{\text{opt}}$/$\beta^{\text{opt}}$ &  $\text{PSA}^{\text{opt}}$ determined by $(o, d)$, and $\beta^{\text{opt}} = |\hat{\beta}^{\text{opt}}|$.\\
  $\hat{\alpha}$/$\alpha$ & PSA determined by $(p_s, o)$, and $\alpha = |\hat{\alpha}|$. \\
  $\hat{\alpha}^{\text{opt}}$/$\alpha^{\text{opt}}$ & $\text{PSA}^{\text{opt}}$ determined by $(p_s, o)$, and $\alpha^{\text{opt}} = |\hat{\alpha}^{\text{opt}}|$.\\
  \bottomrule
  \end{tabular}
\end{table}

We take an example to illustrate the changes of the schedule, path, and service list. We assume the previous service list of PV $p$ is $\{ 1, 2, 3 \}$, and the previous path is $\{ x \rightarrow o_1 \rightarrow o_2 \rightarrow o_3 \rightarrow d_2 \rightarrow d_1 \rightarrow d_3 \}$ where $x$ is the current position of $p$. The change of path has two cases. \emph{First}, the service list does change while the path changes. For example, after a certain time, the path becomes $\{ x' \rightarrow o_2 \rightarrow o_3 \rightarrow d_2 \rightarrow d_1 \rightarrow d_3 \}$ where $x'$ is the new position of $p$. \emph{Second}, both the service list and the new path change. For example, after a certain time, the cloud receives a new request $r_4$, and calculates the new ride-sharing path, and finally assigns $p$ to serve $r_4$. The service list of $p$ becomes $\{ 1, 2, 3, 4 \}$, and the path becomes $\{ x' \rightarrow o_1 \rightarrow o_2 \rightarrow o_3 \rightarrow o_4 \rightarrow d_2 \rightarrow d_1 \rightarrow d_3 \rightarrow d_4 \}$.

\subsection{Problem Formulation and NP-Completeness}

To improve the profits, PVs try to serve more passengers with less travel distance; however, this reduces the user satisfaction. To improve the service quality, passengers want to arrive their destinations as early as possible, i.e., to reduce the waiting time and travel time. To balance the utilities of PVs and passengers, we consider the total utilities of both sides to achieve the best social welfare.

The detour ratio of a request is considered as the passenger QoS. We assume $p$ serves $r$. With respect to request $r$, the actual travel distance $T_r$ is not shorter than $D(o, d)$, the shortest path distance from the origin $o$ to the destination $d$. The detour ratio is defined as $\delta = \left( T_r - D(o, d) \right) \, / \, D(o, d)$. The QoS constraint is then $\delta \leq \Delta$ where $\Delta$ is the maximum detour ratio set by the cloud, which aims at preserving the comfort \cite{zhu2015TrafficEfficiency} of passengers.

The latest arrival time is not considered here for several reasons. 1) It is hard to accurately predict the speed due to traffic congestions, emergencies, and accidents, although some researchers have proposed new solutions \cite{polson2017traffic-prediction}. 2) It is hard to ensure that a passenger can arrive at his/her destination before the latest arrival time he/she set, even taking a car or taxi.

The PVRSP problem is formulated as follows. Given a set of PVs $\mathcal{P}$ on the road networks with the current service lists $\mathcal{L}$ and paths $\mathcal{Q}$, a set of scheduled requests $\mathcal{R}_{s,2}$ waiting to be served, and a set of unscheduled requests $\mathcal{R}_u$, the cloud determines the new service lists $\mathcal{L'}$ and new paths $\mathcal{Q'}$ of all PVs ensuring that all requests can be served from origins to destinations with QoS constraints (the detour ratio is no larger than the maximum value $\Delta$), which aims to reduce the total travel distance of PVs and the waiting time of passengers with QoS guarantee.

Now, we discuss the NP-completeness of the PVRSP problem. It is known to all that the dial-a-ride problem (DARP) \cite{masmoudi2017DARP-GA} is NP-complete \cite{masmoudi2017DARP-GA}, which is a special case of the PVRSP problem. Therefore, the PVRSP problem is also NP-complete.

Here, we mention the differences between the two problems PVRSP and DARP. In the DARP problem, all vehicles are based at a single depot, whereas, in the PVRSP problem, vehicles are distributed at different locations. The DARP problem is an offline/static ride-sharing process, i.e., all requests are known and all passengers stay at their origins waiting to be served, and it does not consider serving other unscheduled requests on the vehicle paths. On the other hand, the PVRSP problem focuses on online/dynamic ride-sharing so that more complex scenarios should be considered, e.g., the precedence constraints between current positions of PVs and the origin-destination pair of each request, and the predetermined match between each scheduled request and the corresponding PV.

\section{Proposed Solution} \label{Sec:Approach}

In this section, we first introduce some heuristics to restrict the potential search areas for PVs, and then describe a key routine of calculating the insertion cost of an origin-destination pair, and then propose an algorithm for online updating PV ride-sharing paths, and finally use an example to illustrate the algorithm.

\subsection{Heuristics} \label{Subsec:Heuristics}

The potential search area (PSA), is an area that points out that the origins or/and destinations of requests in this area are possible to served with ensured QoS (denoted by detour), while the other requests violating QoS constraints are not in this area and should be excluded.

We name $r$ as the \emph{furthest request} if its destination $d$ is the last point on the path of the PV which serves $r$. In the following, we discuss PSA in two conditions, i.e., if the furthest request has been picked up or not. Then, we first describe two cases of PSA for a request, and then derive the PSA of a PV. Here, the PSA of a request and the PSA of a PV are two distinct concepts, while the latter is derived from the former.

\vspace{0.06in}
\emph{1) The furthest request has been picked up}
\vspace{0.06in}

We assume $p$ serves $r$. As shown in Fig.~\ref{Fig:PSATwoPoints1}, if we insert new points (origins or destinations of other requests) between $p$ and $d$. we should restrict an area such that $\delta$ (the detour ratio of $r$) should not exceed its maximum value $\Delta$.

\begin{figure}
  \centering
  \includegraphics[height=0.31\linewidth,width=0.59\linewidth]{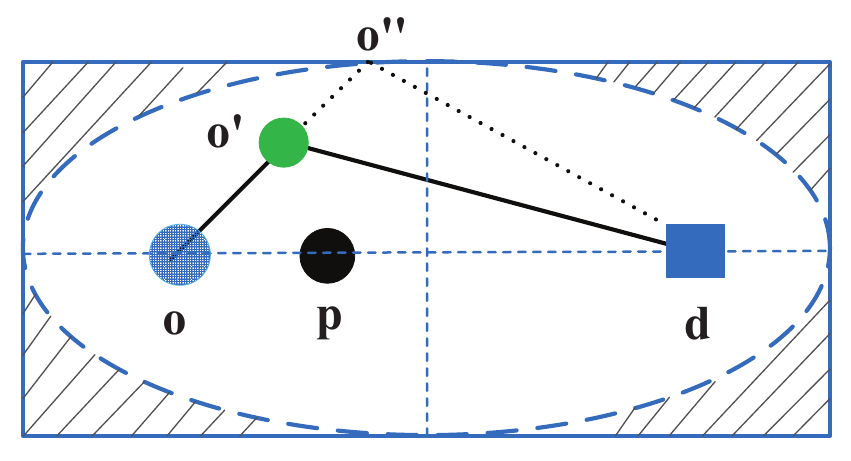}
  \caption{PSA determined by the furthest request $r$ which has been picked up: $\hat{A}_1 = \hat{\beta}$.}
  \label{Fig:PSATwoPoints1}
\end{figure}

\begin{lemm}
Any point (the origin or destination of any request) inserted between $o$ and $d$ should fall in an ellipse determined by $o$ and $d$, otherwise, the QoS will be violated.
\end{lemm}

\begin{proof}
In Fig.~\ref{Fig:PSATwoPoints1}, a point, e.g., $o'$ (origin of $r'$) will be inserted between $o$ and $d$. Let $D(i, \ldots, j)$ and $E(i, \ldots, j)$ denote the shortest path distance and Euclidean distance from the first to the last location among positions $(i, i + 1, \ldots, j)$ respectively. Obviously, $(1 + \Delta) D(o, d) \geq D(o, o', o'', d) \geq E(o, o', o'', d)$, i.e., if we draw an ellipse centering at $o$ and $d$ with the major axis $(1 + \Delta) D(o, d)$, the insertion points should fall in this ellipse. Otherwise, $\delta$ (the detour ratio of $r$) will exceed the maximum value $\Delta$, and the QoS of $r$ will be violated.
\end{proof}

To simplify the computation, instead of an ellipse (denoted by the dashed blue line in Fig.~\ref{Fig:PSATwoPoints1}), we use the corresponding rectangle (denoted by the solid blue rectangle in Fig.~\ref{Fig:PSATwoPoints1}) to check if the points satisfy the QoS constraints. The enlarged area is limited, which is denoted by the shadowed areas in Fig.~\ref{Fig:PSATwoPoints1}. We know $D(o, d) \geq E(o, d)$. The PSA determined by $(o, d)$ is denoted by $\hat{\beta}$. Let $\beta = |\hat{\beta}|$, which is shown by
\begin{eqnarray}
&& \hspace{-0.2in} \beta = (1 + \Delta) \, D(o, d) \, \sqrt{(1 + \Delta)^2 D^2(o, d) - E^2(o, d)}. \nonumber
\end{eqnarray}
Finally, we get the PSA of $p$ if the furthest request $r$ has been picked up by $\hat{A}_1 = \hat{\beta}$. Let $A_1 = |\hat{A}_1|$, therefore, $A_1 = \beta$.

\vspace{0.06in}
\emph{2) The furthest request has not been picked up}
\vspace{0.06in}

\begin{figure}
  \centering
  \includegraphics[height=0.27\linewidth,width=0.85\linewidth]{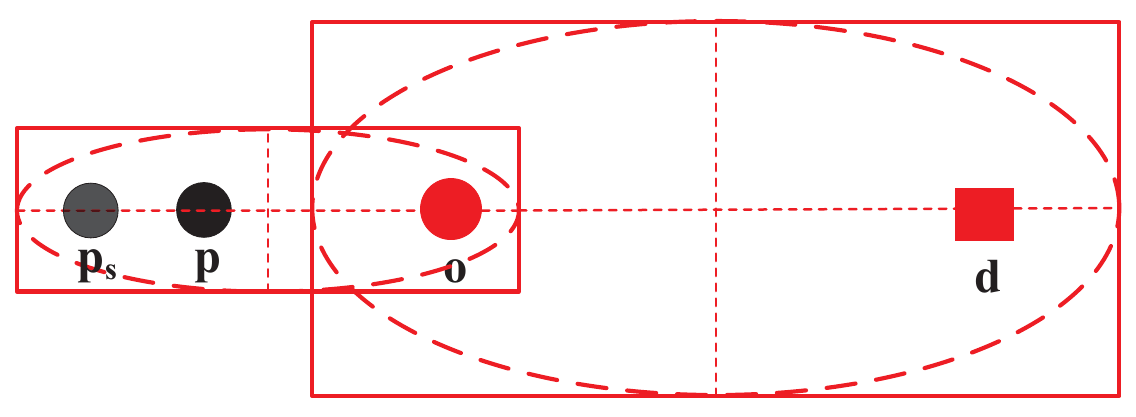}
  \caption{PSA determined by the furthest request $r$ which has not been picked up: $\hat{A}_2 = \hat{\alpha} \bigcup \hat{\beta}$.}
  \label{Fig:PSATwoPoints2}
\end{figure}

Let the schedule time denote the time when schedule of $r$ is confirmed. To limit the waiting time of $r$, we introduce the \emph{buffer distance}, $b$ (km) with a threshold $B$ (km), which denotes the travel distance of $p$ from the schedule time to its pick-up time. Here, we should record the position of PV $p$ at the schedule time of $r$, which is denoted by $p_s$ just as shown in Fig.~\ref{Fig:PSATwoPoints2}. Similarly, we can get the PSA determined by $(p_s, o)$ which is denoted by $\hat{\alpha}$ with $\alpha = |\hat{\alpha}|$, and the PSA determined by $(o, d)$ which is denoted by $\hat{\beta}$ with $\beta = |\hat{\beta}|$. We obtain the PSA of $p$ if the furthest request $r$ has not been picked up denoted by $\hat{A}_2$ with $A_2 = |\hat{A}_2|$ through the union of $\hat{\alpha}$ and $\hat{\beta}$.
\begin{eqnarray}
&& \alpha = W \sqrt{W^2 - E^2(p_s, o)} , \nonumber \\
&& \hat{A}_2 = \hat{\alpha} \bigcup \hat{\beta}. \nonumber
\end{eqnarray}

Next, we discuss the PSA of a PV. We assume that request $r$ with the origin $o$ and the destination $d$ is the furthest request. The request $r'$ with the destination $d'$ is a request which is being served by $p$. We choose the PSA determined by the furthest request (i.e., $r$) not others (e.g., $r'$) as the PSA of PV $p$. The reason is that, if we choose $r'$ to calculate PSA of $p$, we only consider the insertion positions between the current position of $p$ and $d'$, such that the insertion positions between $d'$ and $d$ are ignored.

Let $\hat{A}$ denote the PSA of PV $p$ with $A = |\hat{A}|$, which has two cases and is determined by if the furthest request $r$ has been picked up:
\begin{eqnarray}\label{Eqn:PSA_PV1}
  \hat{A} =
  \begin {aligned}
  \begin{cases}
  \hat{\beta},  \hspace{0.42in} \text{if $r$ has been picked up,} \\
  \hat{\alpha} \bigcup \hat{\beta}, \hspace{0.17in} \text{otherwise,}
  \end{cases}
  \end {aligned}
  \vspace{-0.1in}
\end{eqnarray}
and
\begin{eqnarray}\label{Eqn:PSA_PV2}
  A =
  \begin {aligned}
  \begin{cases}
  \beta,  \hspace{0.42in} \text{if $r$ has been picked up,} \\
  |\hat{\alpha} \bigcup \hat{\beta}|, \hspace{0.10in} \text{otherwise.}
  \end{cases}
  \end {aligned}
  \vspace{-0.1in}
\end{eqnarray}

\subsection{A Key Routine}

Calculating the insertion cost of an origin-destination pair is a key routine in the proposed algorithm. For a new request $r \in R_u$, its origin $o$ has to be visited before its destination $d$. Therefore, if $p$ decides to take request $r$, both $o$ and $d$ need to be inserted into its current path with the precedence constraint being satisfied.

\begin{defi} \label{definition:InsertionCost}
Assume that $r$ is taken by $p$. The insertion cost $\phi_{r, p, i, j}$ at $(i, j)$ is the additional travel distance of $p$ if inserting an origin-destination pair $(o, d)$ of $r$ at the $i^{\text{th}}$ and $j^{\text{th}}$ positions respectively on the path of $p$ with $o$ precedes $d$.
\end{defi}

Let $Q_p = \{\theta_0, \theta_1, \ldots, \theta_K\}$ denote the current path of $p$, where $\theta_0$ is the position of this PV. Let $D(\theta_i, \ldots, \theta_j) = D(\theta_i, \theta_{i+1}) + \ldots, + D(\theta_{j-1}, \theta_j)$ ($i < j$) denote the sum of shortest path distance from $\theta_i$ to $\theta_{i+1}$, $\ldots$ and to $\theta_j$. (\ref{Eqn:Case1})$\sim$(\ref{Eqn:Case4}) describe four cases. The \emph{first} case given by (\ref{Eqn:Case1}) means that $d$ is not the last point of the path, and $o$ immediately precedes $d$. The \emph{second} case given by (\ref{Eqn:Case2}) means that $d$ is the last point, and $o$ immediately precedes $d$. The \emph{third} case given by (\ref{Eqn:Case3}) means that $d$ is the last point, and $o$ does not immediately precede $d$. The \emph{fourth} case given by (\ref{Eqn:Case4}) means that $d$ is not the last point, and $o$ does not immediately precede $d$. Let a function $Q_p$ = \emph{INSERT}$(r, p, i, j)$ returns $Q_p$ (the path of $p$) by inserting the origin $o$ and the destination $d$ of $r$ at the $i^{\text{th}}$ and $j^{\text{th}}$ positions respectively on the path of $p$.

\begin{eqnarray}
\hspace{-0.55in} \phi_{r, p, i, j} &=& D(\theta_{i}, o, d, \theta_{i+1}) - D(\theta_{i}, \theta_{i+1}), \nonumber \\
&&  \hspace{0.1in}  \text{if}~ 0 \leq i\leq K - 1, j = i + 1, \label{Eqn:Case1}\\
 \hspace{-0.54in} \phi_{r, p, i, j} &=& D(\theta_{K}, o, d), \nonumber \\
&& \hspace{0.1in} \text{if}~ i = K, j = i + 1, \label{Eqn:Case2}\\
\hspace{-0.545in} \phi_{r, p, i, j} &=& D(\theta_{i}, o, \theta_{i+1}) + D(\theta_{K}, d) - D(\theta_{i}, \theta_{i+1}), \nonumber \\
&&  \hspace{0.1in}  \text{if}~ 0 \leq i \leq K - 1, j = K + 1, \label{Eqn:Case3}\\
\hspace{-0.54in} \phi_{r, p, i, j} &=& D(\theta_{i}, o, \theta_{i+1}) + D(\theta_{j}, d, \theta_{j+1}) \nonumber \\
&&\hspace{0.01in} - D(\theta_{i}, \theta_{i+1}) - D(\theta_{j}, \theta_{j+1}), \nonumber \\
&&  \hspace{0.1in}  \text{if}~ 0 \leq i \leq K - 1, j \leq K, j \neq i + 1. \label{Eqn:Case4}
\end{eqnarray}

\subsection{Algorithm} \label{Subsec:Algorithm}


The basic idea of the proposed solution is that, we draw a search area for each PV in three cases according to the insertion positions of origin-destination pairs of requests which ensures that serving the requests out of this search area leads to the violation of QoS, so that we only search requests in this limited search area instead of the whole city.

\begin{figure}
  \centering
  \includegraphics[height=0.62\linewidth,width=0.89\linewidth]{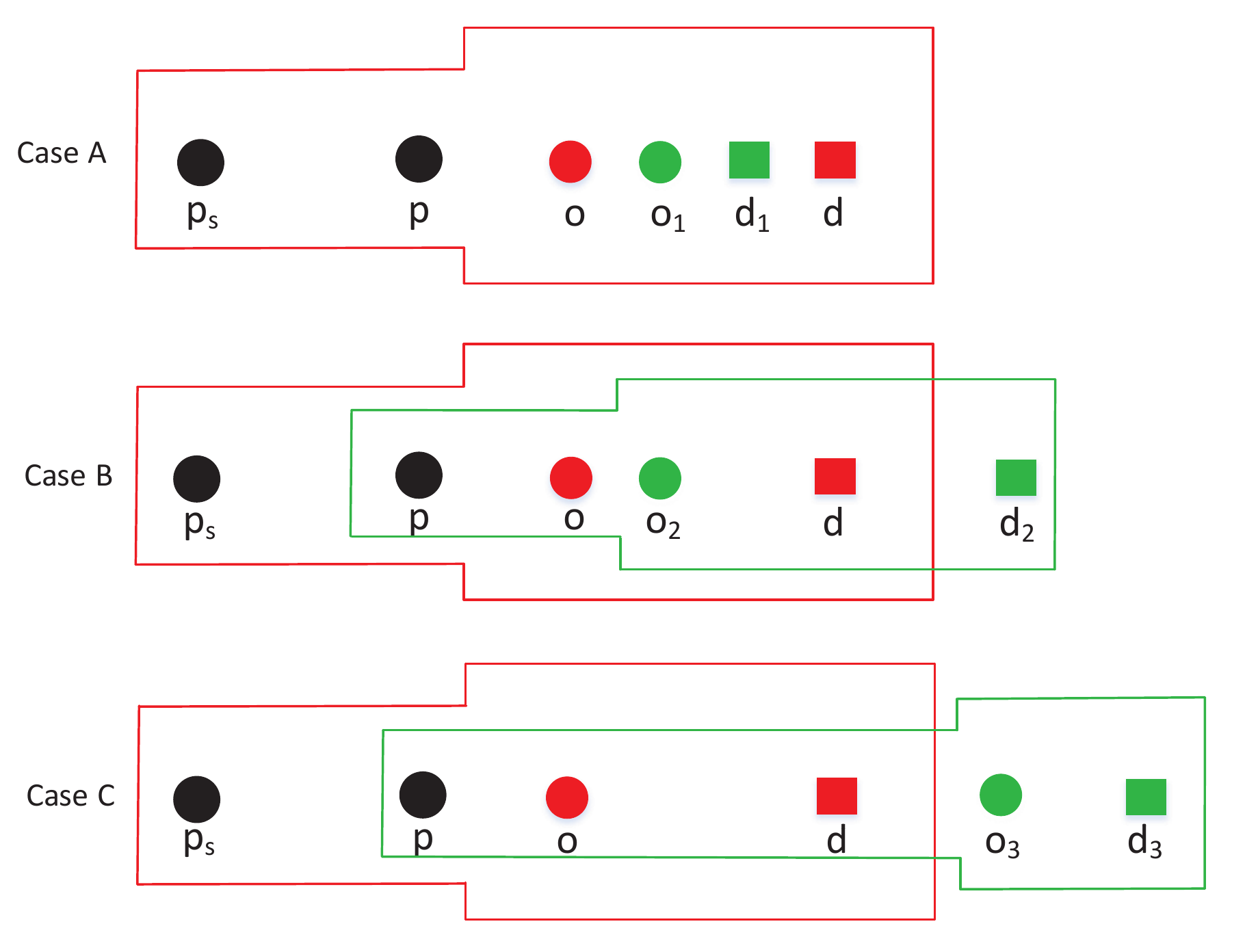}
  \caption{Corresponding PSA in three insertion cases.}
  \label{Fig:CorrespondingPSAThreeInsertionCases}
\end{figure}

Here, we discuss three insertion cases using PSA on the paths of PVs according to the insertion positions of requests. In order to insert the origin-destination pair of a new request $r$, select a position $\theta_i ~ (0 \leq i \leq K)$ of path of $p$ to insert $o$ after $\theta_i$. Then from the positions after $o$ select another position $\theta_j ~ (i + 1 \leq j \leq K + 1 )$ to insert $d$. As shown in Fig.~\ref{Fig:CorrespondingPSAThreeInsertionCases}, there are three cases according to the final insertion positions of the origin and the destination:\\
Case $\mathbb{A}$: None of the origin ($o_1$) and the destination ($d_1$) becomes the last point of the new path. \\
Case $\mathbb{B}$: The destination ($d_2$) becomes the last point of the new path, while the origin ($o_2$) does not precede $d_2$ immediately. \\
Case $\mathbb{C}$: The destination ($d_3$) becomes the last point of the new path, and the origin ($o_3$) precedes the $d_3$ immediately.

In Fig.~\ref{Fig:CorrespondingPSAThreeInsertionCases}, we can see the corresponding PSA in three insertion cases. $d$ is the last point on the path before new requests are inserted. $(o_i, d_i)$ is the origin-destination pair of request $r_i$ ($i \in \{ 1, 2, 3 \}$). The three requests $r_1$, $r_2$, and $r_3$ belong to cases $\mathbb{A}$, $\mathbb{B}$, and $\mathbb{C}$ respectively. The red lines denote the previous PSA of this PV, and the green lines in cases $\mathbb{B}$ and $\mathbb{C}$ denote the new PSA of this PV. Fig.~\ref{Fig:CorrespondingPSAThreeInsertionCases} implies that the PSA in cases $\mathbb{B}$ and $\mathbb{C}$ has changed since the destination of the newly inserted request becomes the last point on the path. Finally, we discuss case $\mathbb{C}$: all points of the previous path of $p$ fall in the PSA $\hat{\alpha}$, which is determined by the current position of $p$ and $o$ since the PSA determined by $o$ and $d$ can be ignored.

If we only consider restricting an area for each PV, there may exist some ``dead zones" such that some unscheduled requests do not satisfy QoS constraints of any PSA of PVs for a long time. To avoid this condition, we introduce a waiting time threshold $W$ for each request. Let $w$ denote the waiting time for request $r$. If $w$ is larger than a threshold $W$, we directly calculate the insertion cost, ignoring the check about the buffer distance and PSA.

\textbf{Algorithm 1} details the proposed solution named as PSA-based path planning algorithm (PSAP). Here, $\delta'$ and $b'$ are the detour ratio and buffer distance of request $r'$ respectively. At some time $t_0$, the cloud selects the unscheduled requests $\mathcal{R}_u$, whose earliest start time is not later than $t_0$. Then the cloud sorts requests $R_u$ in the descending order of their waiting time. In this algorithm, we assume PSA of all PVs have been calculated using (\ref{Eqn:PSA_PV1}) and (\ref{Eqn:PSA_PV2}), which can be inferred by line 28.

The waiting time of requests is balanced in line 2 such that the waiting time of each request will not be too long. Lines 5$\sim$6 check capacity constraints. Lines 9$\sim$14 mean that, this algorithm first checks if the origin and/or destination of the request are/is in PSA, and then calculates the insertion cost in three cases ($\mathbb{A}$, $\mathbb{B}$, and $\mathbb{C}$). In lines 15 and 19, the QoS constraint is checked. Line 21 calculates the minimum insertion cost for PV $p$, and line 22 calculates the minimum insertion cost for all PVs. Lines 23$\sim$26 imply that, if the cloud can find a suitable PV to serve the request, the path and the service list of this PV should be updated, and this request will be removed from the set of unscheduled requests. Lines 27$\sim$28 means that, if destination of the newly inserted request becomes the last point on the path, the PSA of this PV should also be updated, since the PSA is determined by the furthest request.

\vspace{0.05in}
\begin{tabular}[tbp]{lp{0.45\textwidth}}
  \toprule
  \textbf{Algorithm 1}: PSAP \\
  \toprule
  ~~~~\textbf{Input}: $\mathcal{R}_u$, set of unscheduled requests;\\
  ~~~~~~~~~~~~~$\{Q_p\}_{p \in P}$, current paths of PVs;\\
  ~~~~~~~~~~~~~$\{L_p\}_{p \in P}$, current service lists of PVs;\\
  ~~~~\textbf{Output}: $\mathcal{R}'_u$, new set of unscheduled requests;\\
  ~~~~~~~~~~~~~~~$\{L'_p\}_{p \in P}$, new service lists of PVs;\\
  ~~~~~~~~~~~~~~~$\{Q'_p\}_{p \in P}$, new paths of PVs;\\
  01:~Initialization:$\{Q'_p\}\!\leftarrow\!\{Q_p\}$,$\{L'_p\}\!\leftarrow\!\{L_p\}$,$\mathcal{R}'_u\! \leftarrow\!\mathcal{R}_u$; \\
  02:~Sort the unscheduled requests in the descending  \\
  ~~~~order of their waiting time;\\
  03:~\textbf{for} $r \in \mathcal{R}_u$ \textbf{do}\\
  04:~~~\textbf{for} $p \in \mathcal{P}$ \textbf{do}\\
  05:~~~~~\textbf{if} $|L'_p| + n > C$ \textbf{then}\\
  06:~~~~~~~$\phi_{r,p} \leftarrow \infty $;\\
  07:~~~~~\textbf{else} \\
  08:~~~~~~~\textbf{if}~$w \leq W$\\
  09:~~~~~~~~~\textbf{if}~$o \in \hat{A}$ AND $d \in \hat{A}$\\
  10:~~~~~~~~~~~Calculate $\{\phi_{r, p, i, j}\}_{i, j}$ of case $\mathbb{A}$;\\
  11:~~~~~~~~~\textbf{if}~$o \in \hat{A}$ AND $d \notin \hat{A}$\\
  12:~~~~~~~~~~~Calculate $\{\phi_{r, p, i, j}\}_{i, j}$ of case $\mathbb{B}$;\\
  13:~~~~~~~~~\textbf{if}~All points on path of $p$ fall in $\hat{\alpha}$\\
  14:~~~~~~~~~~~Calculate $\{\phi_{r, p, i, j}\}_{i, j}$ of case $\mathbb{C}$;\\
  15:~~~~~~~~~\textbf{if}~$\exists$ $r' \in \{ r \bigcup L_p\}$, $\delta' > \Delta$ \text{OR} $b' > B$\\
  16:~~~~~~~~~~~$\phi_{r, p, i, j} \leftarrow \infty $;\\
  17:~~~~~~~\textbf{else}\\
  18:~~~~~~~~~Calculate $\{\phi_{r, p, i, j}\}_{i, j}$;\\
  19:~~~~~~~~~\textbf{if} $\exists$ $r' \in \{ r \bigcup L_p\}$, $\delta' > \Delta$\\
  20:~~~~~~~~~~~$\phi_{r, p, i, j} \leftarrow \infty $;\\
  21:~~~~~$\phi_{r, p} \leftarrow \min \{\phi_{r, p, i', j'}\}_{i', j'}$;  \\
  22:~~~$\phi_r \leftarrow \phi_{r, p', i'', j''} = \min \{\phi_{r, p}\}_p$; \\
  23:~~~\textbf{if} $\phi_r \neq \infty $ \textbf{then}\\
  24:~~~~~$Q'_{p'} \leftarrow \text{INSERT}(r, p', i'', j'')$; // Update path \\
  25:~~~~~$L'_{p'} \leftarrow L'_{p'} \bigcup \{r\}$; // Update service list  \\
  26:~~~~~$\mathcal{R}'_u \leftarrow \{ \mathcal{R}'_u \backslash r\}$; // Update the set of unsche-\\
  ~~~~~~~~~duled requests\\
  27:~~~~~\textbf{if} $d$ becomes the last point on path of $p'$\\
  28:~~~~~~~Calculate $A'$; // Update PSA of $p'$\\
  \bottomrule
\end{tabular}
\vspace{0.05in}

\subsection{Example}

\begin{figure}[htbp]
  \centering
  \includegraphics[height=0.29\linewidth,width=0.99\linewidth]{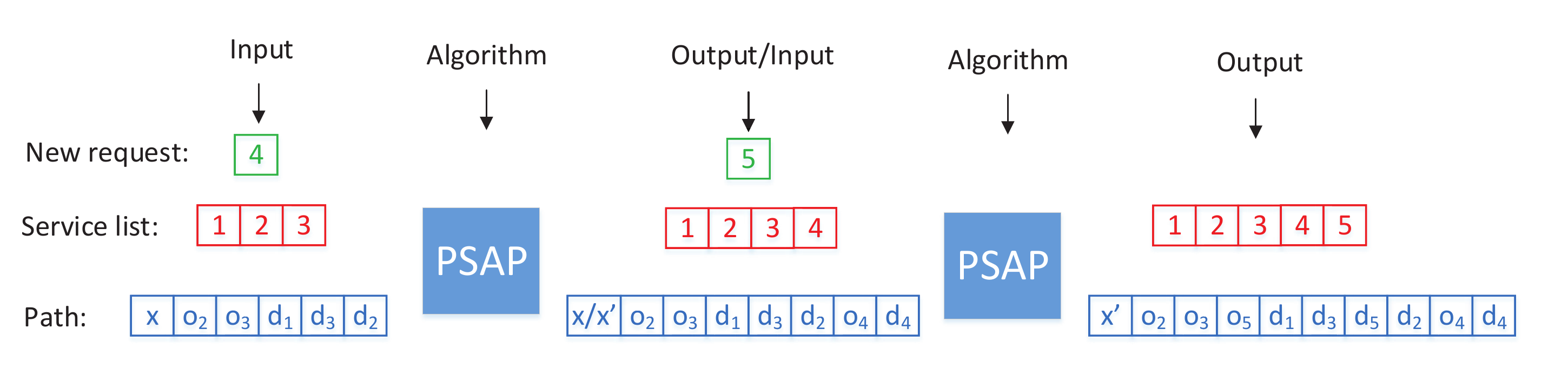}
  \caption{An example of PSAP.}
  \label{Fig:ExampleOfPSAP}
\end{figure}

Fig.~\ref{Fig:ExampleOfPSAP} shows an example to illustrate the PASP algorithm, which has two stages. In the \emph{first} stage, a PV serves three requests $1$, $2$, and $3$, i.e., the service list is $\{ 1, 2, 3 \}$, and the current path is $\{ x \rightarrow o_2 \rightarrow o_3 \rightarrow d_1 \rightarrow d_3 \rightarrow d_2 \}$ where $x$ is the current position of this PV, $o_i$ is the pick-up point of request $i$ and $d_i$ is the drop-off point of request $i$ with $i = \{ 1, 2, 3, 4, 5 \}$. We can see that request $1$ has been picked up and others are still waiting to be served. Now, the cloud receives a new request $4$, while only the case $\mathbb{C}$ satisfies QoS constraints, therefore, the new path becomes $\{ x \rightarrow o_2 \rightarrow o_3 \rightarrow d_1 \rightarrow d_3 \rightarrow d_2 \rightarrow o_4 \rightarrow d_4 \}$ and the new service list becomes $\{ 1, 2, 3, 4 \}$.

In the \emph{second} stage, the position of this PV becomes $x'$, and request $1$ still has not been dropped off, and requests $2$, $3$, and $4$ have not been picked up, and now the cloud receives a new request $5$. PSAP finds that only case $\mathbb{A}$ satisfies QoS constraints, and the path $\{ x' \rightarrow o_2 \rightarrow o_3 \rightarrow o_5 \rightarrow d_1 \rightarrow d_3 \rightarrow d_5 \rightarrow d_2 \rightarrow o_4 \rightarrow d_4 \}$ achieves the minimum insertion cost, which becomes the new path and $\{ 1, 2, 3, 4, 5 \}$ becomes the new service list.

\section{Performance Analysis} \label{Sec:Analysis}

In this section, we first compare our solution with the exhaustive search, and then present properties about the gap between PSA and optimal PSA ($\text{PSA}^{\text{opt}}$) on the performance, and finally we analyze the reduction ratio of computational complexity.

\subsection{PSAP vs. Exhaustive Search}

The exhaustive search (ES) method has to try each insertion position in the path of each PV, i.e., ignoring the process of checking requests which insertion case they belong to and directly calculating the insertion cost through (\ref{Eqn:Case1})$\sim$(\ref{Eqn:Case4}). The other process is the same as that of PSAP. We see that ES can be revised from PSAP: lines (9)$\sim$(14) are replaced by ``Calculate $\{\phi_{r, p, i, j}\}_{p, i, j}$ through (\ref{Eqn:Case1})$\sim$(\ref{Eqn:Case4})" by trying each insertion position over all PVs, such that a large amount of computation is needed in ES. We see that PSAP and ES have the same service quality performance while different computational complexity. PSAP can greatly improve the computational efficiency by restricting the search areas compared with ES since PSAP only tries a part of requests and others are excluded.

Here, we present more details in PSAP to show its advantages. As shown in Fig.~\ref{Fig:PSATwoPoints1}, some requests which violate QoS are excluded, since the detour of at least a request is larger than $\Delta$ if they are not inserted in PSA. The scenario in Fig.~\ref{Fig:PSATwoPoints2} is similar. The travel distance of the PV from schedule time to pick-up time is limited, and the detour is limited. Therefore, the QoS of requests is preserved. However, all the above scenarios which violating QoS will be calculated in ES, leading to a large amount of computation.

\subsection{Gap between PSA and $\text{PSA}^{\text{opt}}$}

\begin{defi} \label{definition:OptimalPSA-P}
If the furthest request $r$ has been picked up, the optimal PSA ($\text{PSA}^{\text{opt}}$) of PV $p$ is the ellipse centering at $(o, d)$, otherwise, is the union of two ellipses centering at $(p_s, o)$ and $(o, d)$ respectively, just as discussed in Section \ref{Subsec:Heuristics}.
\end{defi}

In practice, we use PSA instead of $\text{PSA}^{\text{opt}}$ since the enlarged area is limited. We take two examples to discuss $\text{PSA}^{\text{opt}}$. In Fig.~\ref{Fig:PSATwoPoints1}, $\text{PSA}^{\text{opt}}$ is the ellipse and the PSA is the corresponding rectangle. In Fig.~\ref{Fig:PSATwoPoints2}, $\text{PSA}^{\text{opt}}$ is the union of two ellipses and the PSA is the union of two corresponding rectangles. We see that the PSA is always a litter larger than $\text{PSA}^{\text{opt}}$.

Here, we discuss the gap between PSA and $\text{PSA}^{\text{opt}}$ to show how much computational complexity is enlarged using PSA than the latter one. Let $\hat{A}$ denotes PSA of $p$ with $A = |\hat{A}|$, and $\hat{A}^{\text{opt}}$ denotes $\text{PSA}^{\text{opt}}$ of $p$ with $A^{\text{opt}} = |\hat{A}^{\text{opt}}|$, and let $\eta$ denote the ratio of PSA and $\text{PSA}^{\text{opt}}$, which is formulated as (\ref{Eqn:eta}). Obviously, $\eta > 1$. In practice, we hope the gap between PSA and $\text{PSA}^{\text{opt}}$ is smaller since this means more computational complexity is saved, i.e., the smaller $\eta$ is, the more efficient the PSAP will be.
\begin{equation}\label{Eqn:eta}
 \eta = \frac{A}{A^{\text{opt}}}.
\end{equation}

\begin{theo}\label{theorem:RatioPSA_OD}
With respect to any PV $p$, if the furthest request $r$ has been picked up, the ratio of PSA and $\text{PSA}^{\text{opt}}$ $\eta = \frac{A}{A^{\text{opt}}} = \frac{4}{\pi}$.
\end{theo}

Theorem \ref{theorem:RatioPSA_OD} can be obtained according to the area of the ellipse and corresponding rectangle. Next, let's discuss the gap between PSA and $\text{PSA}^{\text{opt}}$ of any PV $p$. Let $r$ denote the furthest request on path of $p$.

From Fig.~\ref{Fig:PSATwoPoints2}, we know that, if the furthest request has not been picked up, the PSA (or $\text{PSA}^{\text{opt}}$) of the PV is obtained by union of two PSA (or $\text{PSA}^{\text{opt}}$) determined by the request. Now let $\hat{\alpha}^{\text{opt}}$ denote $\text{PSA}^{\text{opt}}$ determined by $(p_s, o)$ with $\alpha^{\text{opt}} = |\hat{\alpha}^{\text{opt}}|$, and $\beta^{\text{opt}}$ denote $\text{PSA}^{\text{opt}}$ determined by $(o, d)$ with $\beta^{\text{opt}} = |\hat{\beta}^{\text{opt}}|$, and $\hat{A}^{\text{opt}}$ denote $\text{PSA}^{\text{opt}}$ of PV $p$ with $A^{\text{opt}} = |\hat{A^{\text{opt}}}|$. Let $\hat{A}^{\text{opt}} = \hat{\alpha}^{\text{opt}} \bigcup \hat{\beta}^{\text{opt}}$. Let $\hat{\mu} = \hat{\alpha} \bigcap \hat{\beta}$ with $\mu = |\hat{\mu}|$ and $\hat{\nu} = \hat{\alpha}^{\text{opt}} \bigcap \hat{\beta}^{\text{opt}}$ with $\nu = |\hat{\nu}|$. Clearly, $\hat{\mu} \supseteq \hat{\nu}$, and $\mu \geq \nu$.

\begin{theo}\label{theorem:RatioPSA_pod_PV}
With respect to any PV $p$, no matter if the furthest request $r$ has been picked up or not, $\eta = \frac{A}{A^{\text{opt}}} \in \left[ \max \big( 1, \frac{4}{\pi} + \frac{4 \nu - \pi \mu}{\pi (\alpha^{\text{opt}} + \beta^{\text{opt}})} \big), \frac{4}{\pi} + \frac{(4 - \pi) \mu}{\pi (\alpha^{\text{opt}} + \beta^{\text{opt}} - \nu)} \right]$.
\end{theo}

\begin{proof}

From (\ref{Eqn:PSA_PV1}), (\ref{Eqn:PSA_PV2}), and Theorem \ref{theorem:RatioPSA_OD}, we know that $\eta$ is involved by the state of the furthest request. Now, we discuss the \emph{first} case, i.e., the furthest request $r$ has not been picked up. From Theorem \ref{theorem:RatioPSA_OD}, we know that,
\begin{equation}
\frac{\alpha}{\alpha^{\text{opt}}} = \frac{\beta}{\beta^{\text{opt}}} = \frac{4}{\pi}. \nonumber
\end{equation}
Considering of the definition of $\mu$ and $\nu$, we get the following:
\begin{eqnarray}
&& A = \alpha + \beta - \mu, \nonumber \label{Eqn:Analysis-2}\\
&& A^{\text{opt}} = \alpha^{\text{opt}} + \beta^{\text{opt}} - \nu, \nonumber  \\
&& \eta = \frac{\alpha + \beta - \mu}{\alpha^{\text{opt}} + \beta^{\text{opt}} - \nu}. \nonumber
\end{eqnarray}
And then $\eta$ is reformulated as
\begin{equation}
\eta = \frac{4}{\pi} + \frac{4 \nu - \pi \mu}{\pi (\alpha^{\text{opt}} + \beta^{\text{opt}} - \nu)}. \nonumber
\end{equation}
Considering that $\eta > 1$, we get
\begin{eqnarray}
&& \hspace{-0.3in} \max \big( 1, \frac{4}{\pi} + \frac{4 \nu - \pi \mu}{\pi (\alpha^{\text{opt}} + \beta^{\text{opt}})} \big)\!\leq\!\eta\!\leq\!\frac{4}{\pi} + \frac{(4 - \pi) \mu}{\pi (\alpha^{\text{opt}} + \beta^{\text{opt}} - \nu)}. \nonumber
\end{eqnarray}

Then, we discuss the \emph{second} case, i.e., the furthest request $r$ has been picked up. Obviously, $\mu = \nu  = 0$, and $\eta = \frac{4}{\pi}$. So the above inequality about the scope of $\eta$ is still workable.

Finally, Theorem \ref{theorem:RatioPSA_pod_PV} is proved.
\end{proof}

\subsection{Reduction Ratio of Computational Complexity}

Reduction ratio of computational complexity (RRCC) includes three cases: $\mathbb{A}$, $\mathbb{B}$, and $\mathbb{C}$. Let $M_\mathbb{A}$, $M_\mathbb{B}$, and $M_\mathbb{C}$ denote the number of insertion times of origin-destination pairs in PSAP in cases $\mathbb{A}$, $\mathbb{B}$, and $\mathbb{C}$. Let $N_\mathbb{A}$, $N_\mathbb{B}$, and $N_\mathbb{C}$ denote the number of insertion times of origin-destination pairs in ES in three cases. Let $\psi_{\mathbb{A}}$, $\psi_{\mathbb{B}}$, and $\psi_{\mathbb{C}}$ denote RRCC in three cases. For example, in case $\mathbb{A}$, we record $M_\mathbb{A}$ and $N_\mathbb{A}$, and then $\psi_{\mathbb{A}}$ is obtained by
\begin{equation}
\psi_{\mathbb{A}} = (N_\mathbb{A} - M_\mathbb{A}) \, / \, N_\mathbb{A}. \nonumber
\end{equation}
$\psi_{\mathbb{B}}$ and $\psi_{\mathbb{C}}$ can be obtained in a similar way. In this subsection, we discuss the estimation for RRCC since it directly affects the computational complexity of PSAP: the larger RRCC is, the more efficient PSAP will be.

Computing PSA involves several square operations, and computing if the origin or destination is in PSA only involves linear operations. Both of them have less computational complexity than the shortest path algorithm such as Dijkstra algorithm and calculating the minimum insertion cost \cite{zhu2016PublicVehicle}. Therefore, the cost of computing PSA and computing if the origin or destination is in PSA can be ignored when we analyze RRCC.

Let $S$ denote the area of the whole road network. If the PSA is smaller, more requests will be excluded, and the PSAP solution will be more efficient. For example, now we insert an origin-destination pair $(o', d')$ of request $r'$ based on PSA $\hat{A}$. In case $\mathbb{A}$, both $o'$ and $d'$ are in the PSA $\hat{A}$.

We assume that the event $o' \in \hat{A}$ is independent with another event $d' \in \hat{A}$. In case $\mathbb{A}$, the probability of $o', d' \in \hat{A}$ is $\frac{A^2}{S^2}$. In case $\mathbb{B}$, the probability of $o' \in \hat{A}$ is $\frac{A}{S}$. In case $\mathbb{C}$, make sure that all the points on the path should fall in the PSA determined by the current position of $p$ and $o$.

Generally, the execution time of case $\mathbb{A}$ is more than case $\mathbb{B}$, and the execution time of case $\mathbb{C}$ is the minimum among three cases, which can be inferred from Section \ref{Subsec:Algorithm}. The execution time of case $\mathbb{C}$ is not meaningful since it needs at most one calculation for one request. We focus on mathematical expectation of RRCC in cases $\mathbb{A}$ and $\mathbb{B}$, which are obtained by
\begin{eqnarray}
&& \mathbf{E} (\psi_{\mathbb{A}}) = 1 - \frac{A^2}{S^2}, \label{Eqn:psi} \\
&& \mathbf{E} (\psi_{\mathbb{B}}) = 1 - \frac{A}{S}, \label{Eqn:psi2}
\end{eqnarray}
where $\mathbf{E}$ means the mathematical expectation.

On the current path of $p$ $Q_p = \{\theta_0, \theta_1, \ldots, \theta_K\}$, we insert $o$ and $d$. Let $i$ be the insertion position of $o$, and we get the following:
\begin{eqnarray}
&& N_\mathbb{A} = \sum_{i = 1}^{K - 1} (K - i) = \frac{1}{2} K (K - 1), \label{Eqn:N_A}\\
&& N_\mathbb{B} = K - 1, \label{Eqn:N_B} \\
&& N_\mathbb{C} = 1. \label{Eqn:N_C}
\end{eqnarray}

The total reduced number of insertion times of origin-destination pairs using PSAP compared with ES is denoted by $I$, and we obtain its mathematical expectation by
\begin{eqnarray}
&& \hspace{-0.53in} \mathbf{E} (I) = N_\mathbb{A} \, \mathbf{E} (\psi_{\mathbb{A}}) + N_\mathbb{B} \, \mathbf{E} (\psi_{\mathbb{B}}) + N_\mathbb{C} \, \mathbf{E} (\psi_{\mathbb{C}}) \nonumber \\
&& \hspace{-0.21in} \approx N_\mathbb{A} \, \mathbf{E} (\psi_{\mathbb{A}}) + N_\mathbb{B} \, \mathbf{E} (\psi_{\mathbb{B}})  \nonumber \\
&& \hspace{-0.21in} = \frac{1}{2} K (K - 1) \left( 1 - \frac{A^2}{S^2} \right) + (K - 1) \left( 1 - \frac{A}{S} \right), \label{Eqn:M}
\end{eqnarray}
where case $\mathbb{C}$ can be ignored since $N_\mathbb{C}$ is much smaller than $N_\mathbb{A}$ and $N_\mathbb{B}$.

Obviously, $\frac{A^2}{S^2} \leq \frac{A}{S} \leq 1$ and we get $\mathbf{E} (\psi_{\mathbb{A}}) \geq \mathbf{E} (\psi_{\mathbb{B}})$. From (\ref{Eqn:N_A})$\sim$(\ref{Eqn:N_C}), we know $N_\mathbb{A} \geq N_\mathbb{B} \geq N_\mathbb{C}$. From (\ref{Eqn:M}), we get that, cases $\mathbb{A}$ and $\mathbb{B}$ almost determine RRCC since generally $N_\mathbb{A} \, \mathbf{E} (\psi_{\mathbb{A}}) \geq N_\mathbb{B} \, \mathbf{E} (\psi_{\mathbb{B}}) \geq N_\mathbb{C} \, \mathbf{E} (\psi_{\mathbb{C}})$.

\section{Performance Evaluation} \label{Sec:Performance}

In this section, we first describe the simulation settings, and then present the results to evaluate the performance of our proposed solution.

\subsection{Simulation Settings}

Simulation settings include: PV setting, privately owned electric vehicle (POEV) setting, road map data, taxi data, and parameter setting.

\textbf{PV Setting}: We use electric vehicles, Tesla Model-S \cite{TeslaModelS} to study the transportation patterns of PVs, although they are not self-driving vehicles now. The number of seats of each PV is 5. Assuming that all PVs travel along the shortest path between any two positions (origins or destinations of requests) with the identical speed 30 km/h. The initial positions of PVs are randomly distributed in the road network which follows a uniform distribution. PVs always travel to serve passengers as long as new requests are assigned, and PVs stay at their positions if no requests are assigned to them.

\textbf{POEV Setting}: We also use the Tesla Model-S \cite{TeslaModelS} to explore the traffic characteristics of POEVs. The settings of POEVs are the same as PVs except that POEVs only provide the origin-to-destination ride service along the shortest path instead of ride-sharing services. Clearly, the number of POEVs is the half the number of requests if each POEV is used twice a day, which is much larger than that of PVs.

\begin{figure}
  \centering
  \includegraphics[height=0.25\linewidth,width=0.98\linewidth]{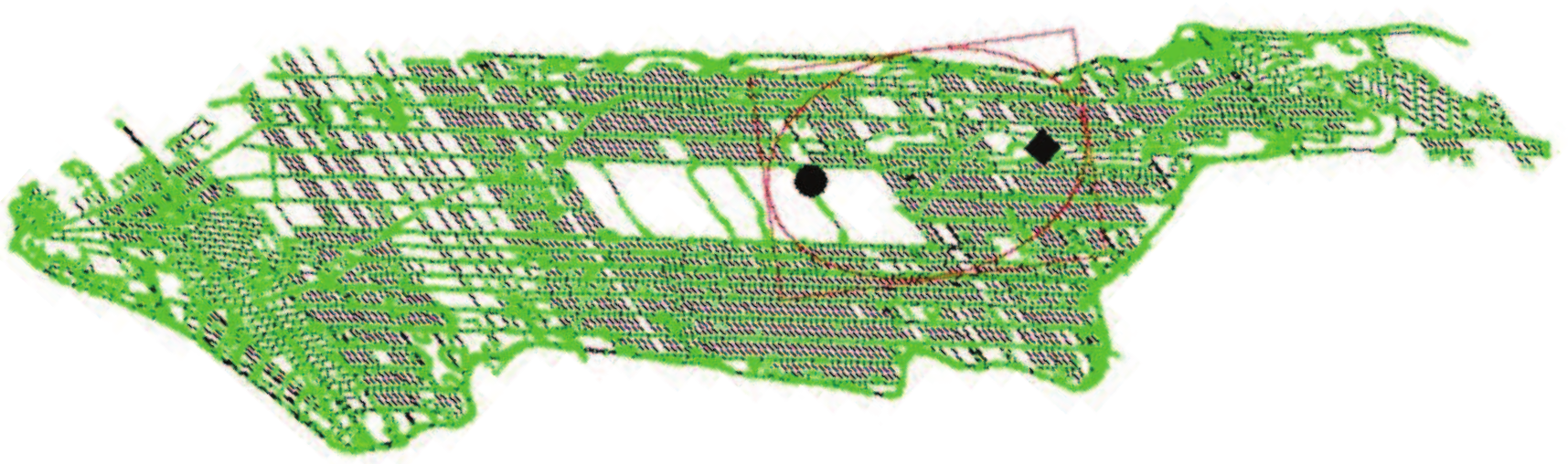}
  \caption{Manhattan in New York City.}
  \label{Fig:Manhattan}
\end{figure}

\textbf{Road Map Data}: We use the Manhattan road network with 60 km$^2$ in New York City to perform PSAP and ES, which is depicted by Fig.~\ref{Fig:Manhattan} where black lines denote road segments, and green points denote nodes. The road map is extracted through the openstreetmap \cite{Openstreetmap} with six types of ways: primary, secondary, tertiary, motorway, motorway\_link, and residential. Finally, 3,900 ways and 29,792 nodes are filtered.

The black circle and black square in Fig.~\ref{Fig:Manhattan} are the origin and the destination of a request respectively. We assume that this request is the furthest and is just picked up by a PV at the origin, and the corresponding PSA and $\text{PSA}^{\text{opt}}$ are denoted by the red rectangle and the red ellipse respectively. We can see that, in case $\mathbb{A}$, the search area denoted by the PSA is much smaller than the whole road network, which implies that the search space is largely reduced and the computational efficiency is improved.

\begin{figure*}[htbp]
  \begin{minipage}[b]{0.33\linewidth}
    \centering
    \includegraphics[height=0.80\linewidth,width=0.99\linewidth]{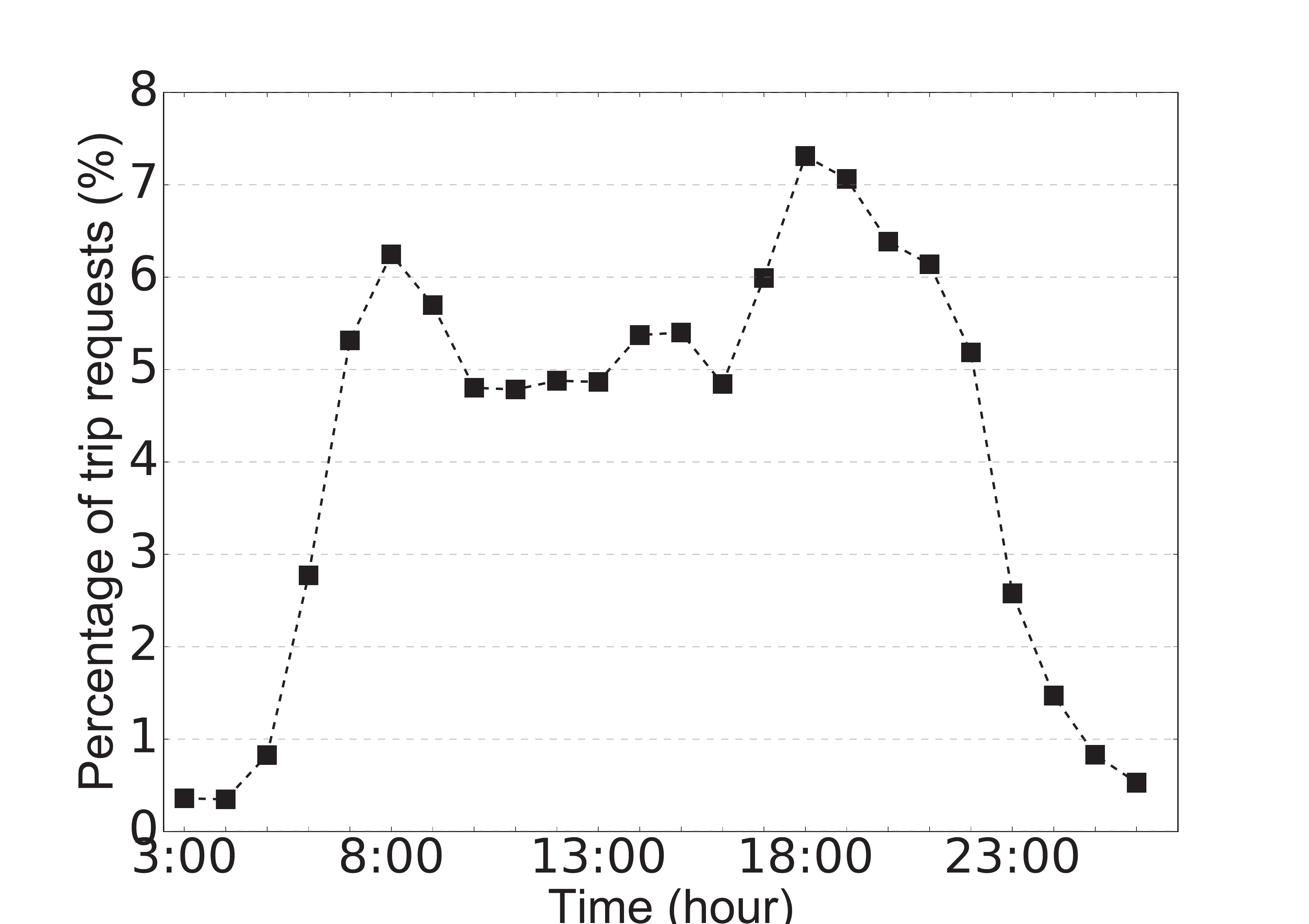}
    \caption{Distribution of requests in one day.}
    \label{Fig:DistributionOfRrequestsOneDay}
  \end{minipage}  
  \begin{minipage}[b]{0.33\linewidth}
    \centering
    \includegraphics[height=0.80\linewidth,width=0.99\linewidth]{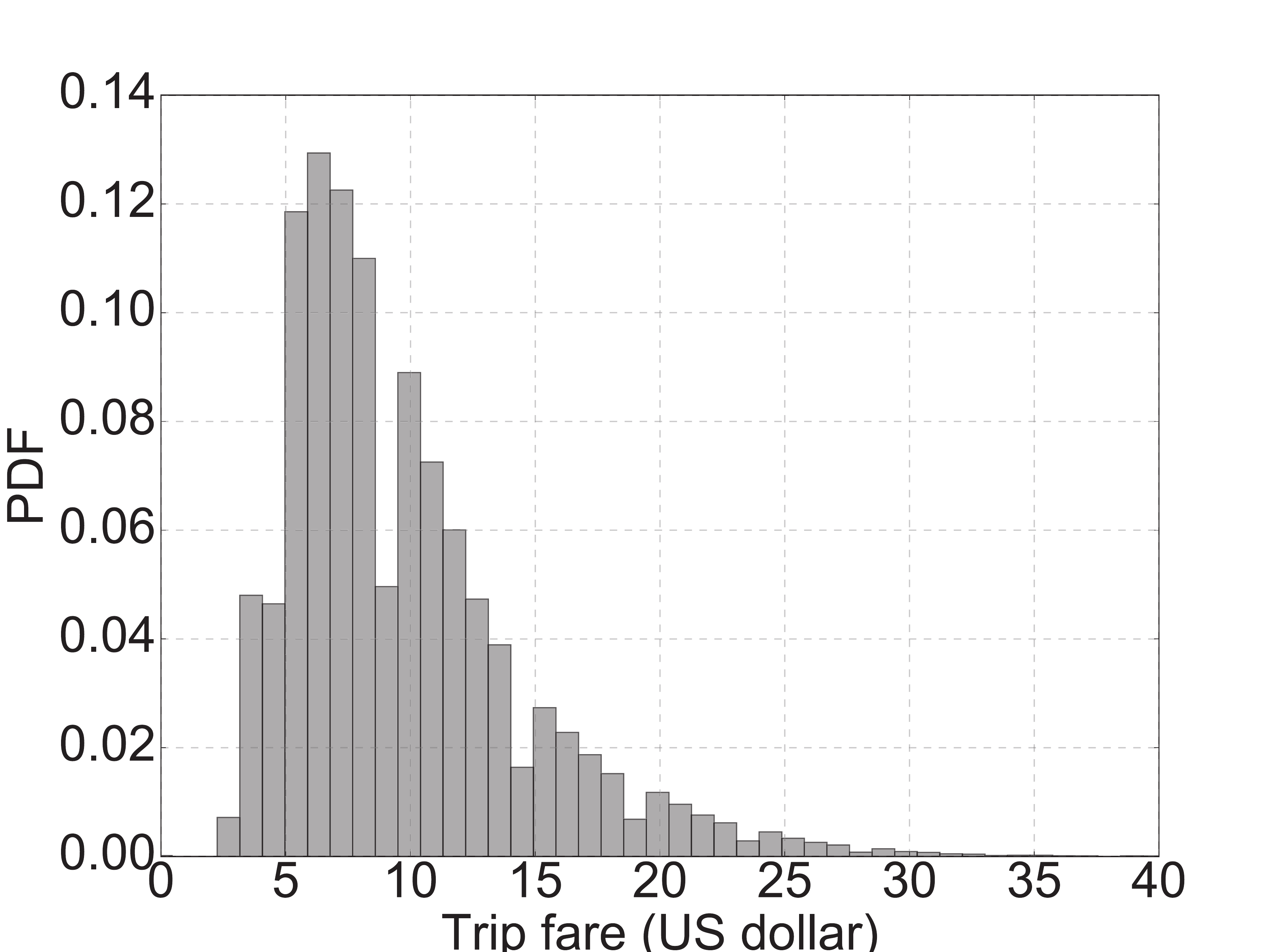}
    \caption{PDF of trip fare.}
    \label{Fig:MahattanTaxiDataFare_PDF}
  \end{minipage}  
  \begin{minipage}[b]{0.33\linewidth}
    \centering
    \includegraphics[height=0.80\linewidth,width=0.99\linewidth]{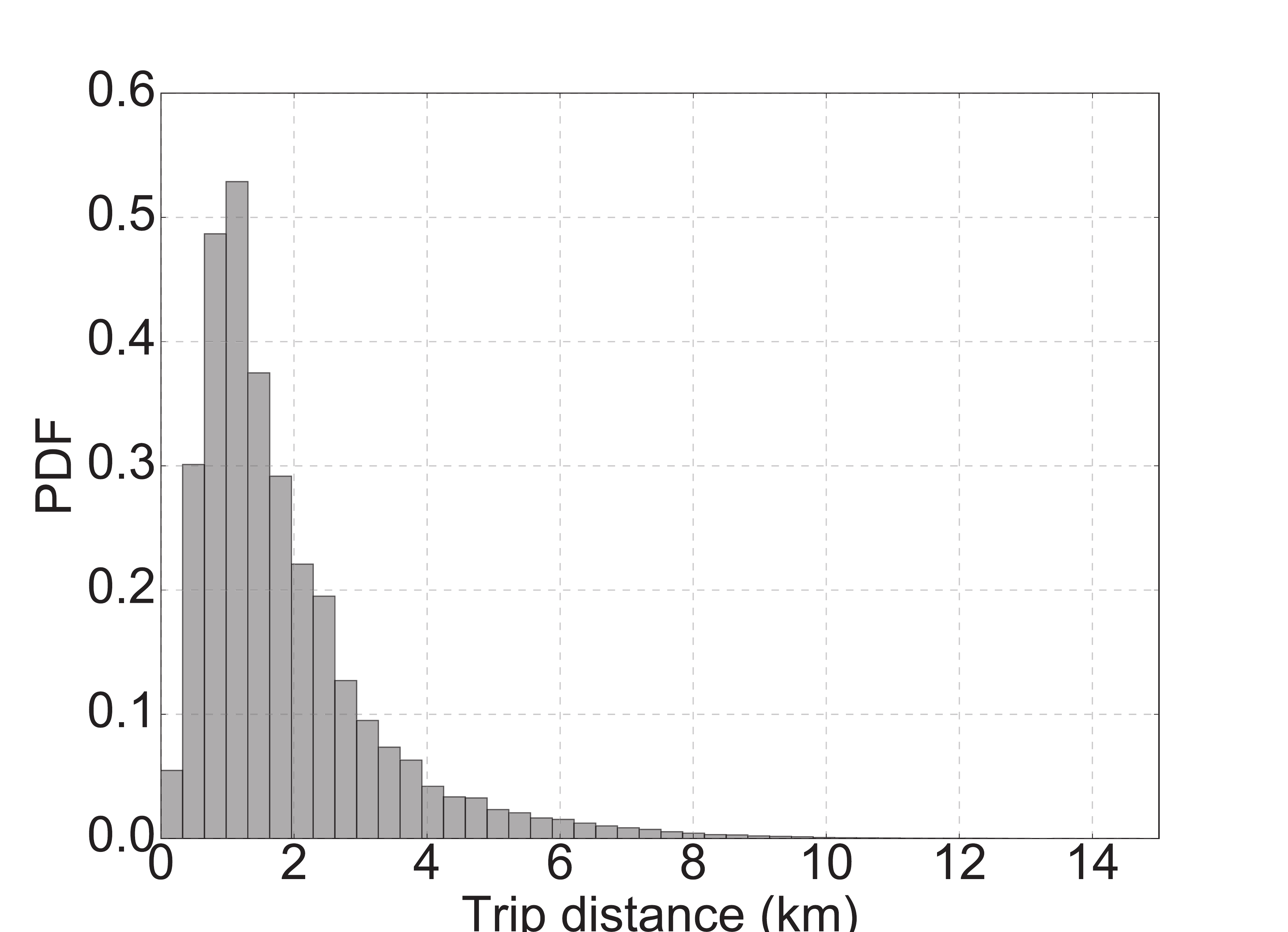}
    \caption{PDF of trip distance.}
    \label{Fig:MahattanTaxiDataDistTrip_PDF}
  \end{minipage}
\end{figure*}

\textbf{Taxi Data}: We use the taxi data set (yellow records) of the New York City on March 1, 2016 \cite{NewYorkCityTaxiData}. Each record contains several useful fields for our study, including pick-up time, drop-off time, trip distance, latitudes/longitudes of origins, latitudes/longitudes of destinations, fares, tax, tip, and total payment. Then the trip records with origin-destination pairs falling in Mahattan are selected.

Let's discuss three traffic characteristics in the Mahattan taxi data set. The \emph{first} one is the distribution of these trip requests (only in Manhattan) in each hour on March 1 (Tuesday), 2016, which is shown by Fig.~\ref{Fig:DistributionOfRrequestsOneDay}. The busiest time occurs at 8:00 and 18:00 respectively in morning and evening rush hours. The \emph{second} one is the PDF of trip fare, which is presented by Fig.~\ref{Fig:MahattanTaxiDataFare_PDF}. We know that most of the trip fare is between 5$\sim$15 US dollars. The \emph{third} one is the PDF of trip distance, which is depicted by Fig.~\ref{Fig:MahattanTaxiDataDistTrip_PDF}. We can see that most of the trip distance is less than 6 km.

We assume the pick-up time of each request is equal to the earliest start time. All PVs travel as the ride-sharing path computed by the cloud, and stop at the destinations of requests for the next schedule task. To make the performance more stable, we select the requests with Euclidean distance no less than 3 km, and then the total number of trip requests of the day in Mahattan is 75,014, and the number is reduced to 3,112 in the period 11:00$\sim$11:59.

\textbf{Parameter Setting}: Passengers prefer low detour ratio, which reflects the QoS of passengers and will be widely accepted if the value is not larger than the maximum value $\Delta = 0.2$. The waiting time threshold $W$ (minute) is 4. If $W$ is too large, some requests which fall in dead zones may have to wait for long time, however, if it is too small, the effect of PSAP is not clear since it directly calculate the insertion cost ignoring the check about PSA. The buffer distance threshold $B$ (km) is set to 6. If $B$ is set too large, the PSA of PVs may be too large just as the whole city and PSAP may spend more time on computing the paths of PVs, while if it is too small, most requests will fall in the limited search area and they will be assigned only after the waiting time is larger than the threshold $W$.

We implement PSAP and ES through a computer with an Intel Core-i7 (3.4 GHz and 32 GB of RAM) using C++ under Windows OS. In the 24-hour simulations, the time begins at 3:00 for two reasons. 1) This time is one of the most important shift handover time in taxi companies in many cities such as Shenzhen \cite{tian2016ChargingStationRecommendation}. 2) The number of trip requests of this time is almost the minimum of one day in multiple cities such as Shanghai \cite{zhu2016PublicVehicle}, Tokyo \cite{atasoy2015FMOD}, and New York City \cite{NewYorkCityTaxiData}.

\subsection{Results}

\begin{figure*}[htbp]
  \begin{minipage}[b]{0.33\linewidth}
    \centering
    \includegraphics[height=0.80\linewidth,width=0.99\linewidth]{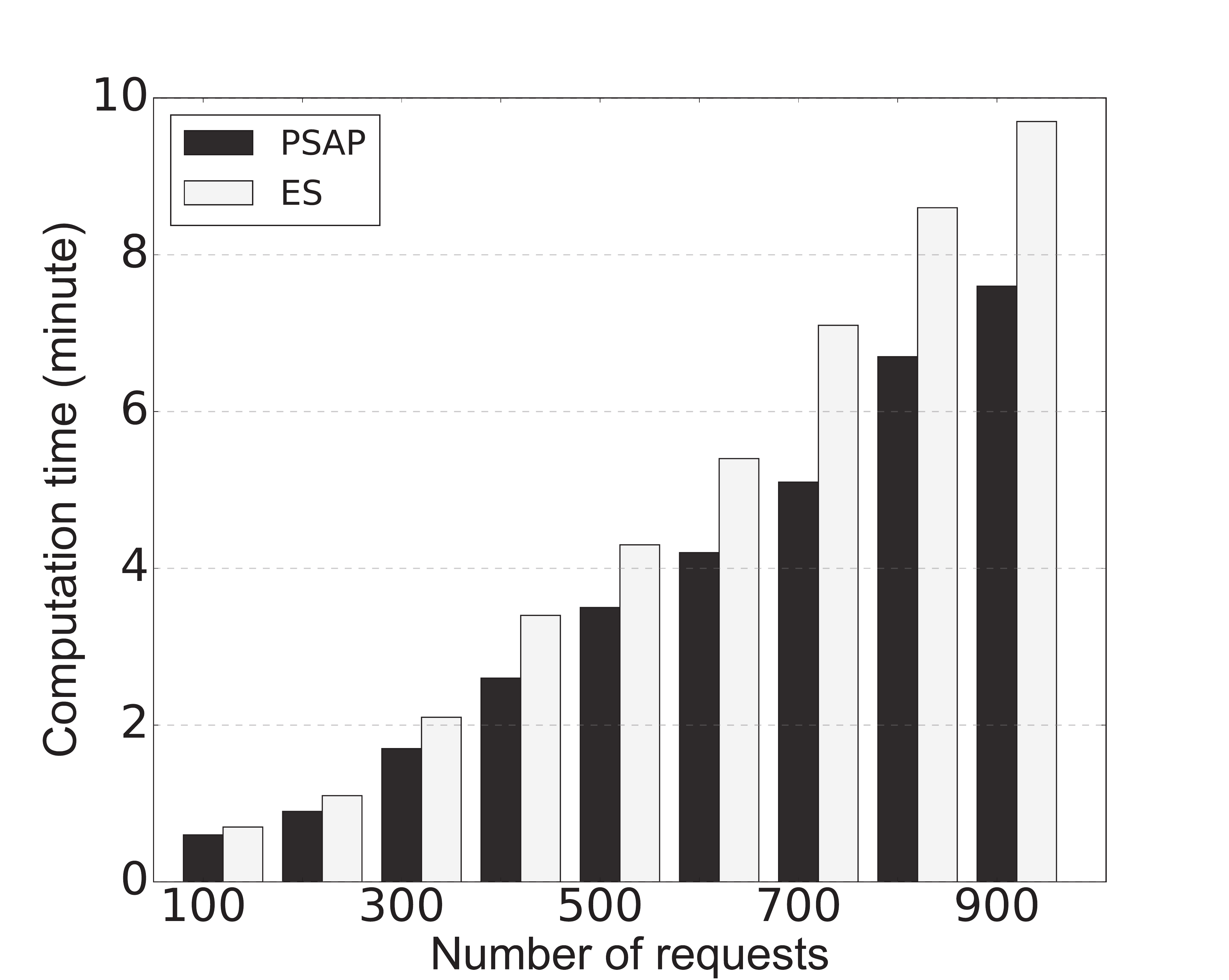}
    \caption{Computation time comparison.}
    \label{Fig:ComputationTimeComparison}
  \end{minipage}
  \begin{minipage}[b]{0.33\linewidth}
    \centering
    \includegraphics[height=0.80\linewidth,width=0.99\linewidth]{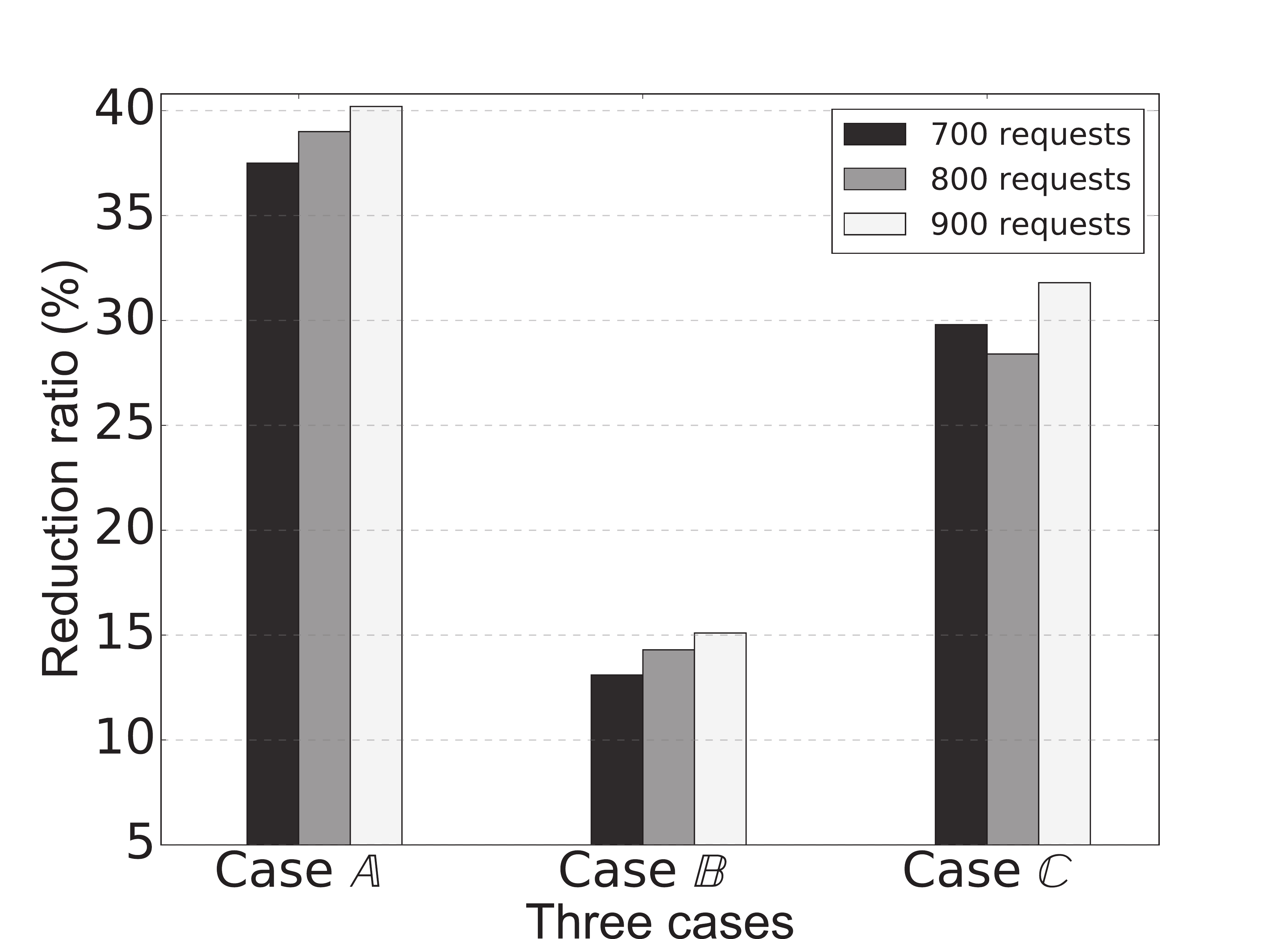}
    \caption{RRCC in three cases.}
    \label{Fig:ReductionRatioComputation}
  \end{minipage}
  \begin{minipage}[b]{0.33\linewidth}
  \centering
  \includegraphics[height=0.80\linewidth,width=0.99\linewidth]{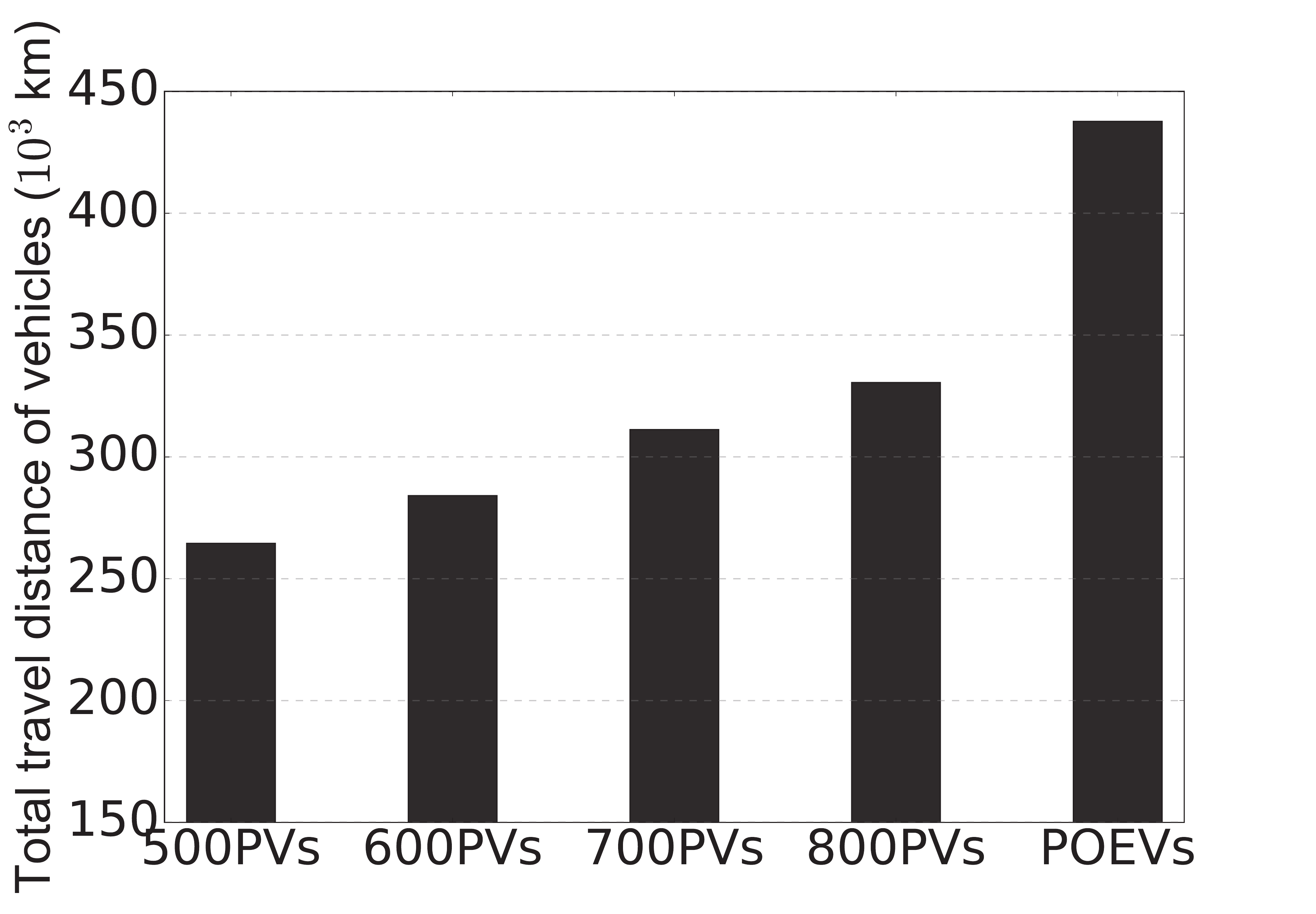}
  \caption{Total travel distance comparison.}
  \label{Fig:TravelDistanceComparison}
  \end{minipage}
\end{figure*}

Here, we introduce two metrics: sharing rate and saved travel distance. 1) The sharing rate should imply the traffic sharing scenario in PV systems, and is denoted by the following
\begin{equation}
\text{sharing rate} = \frac{\sharp ~ \text{of requests being served}}{\sharp ~ \text{of moving vehicles}}.
\end{equation}
We see that the sharing rate of POEVs does not change and is always 1 since there is no sharing among different requests. 2) The saved travel distance is denoted by the sum of the shortest path distance of each origin-destination pair of requests minus the total travel distance of PVs. We see that the saved travel distance implies the amount of saved energy of transportation systems.

We know that PSAP and ES have the same service quality performance, therefore, we compare the computational complexity of PSAP and ES, and then present the traffic performance of PVs using PSAP and POEVs to explore their transportation patterns.

The simulation results are presented in six items: computation time comparison, RRCC, total travel distance comparison, utilization rate, sharing rate, and saved travel distance. And finally, the energy efficiency and particular features of PV systems are discussed.

The \emph{first} metric is the computation time, which is shown by Fig.~\ref{Fig:ComputationTimeComparison} where the number of requests varies from 100 to 900 during 11:00$\sim$11:59 and the number of PVs is 70. We get that, when the number of requests increases, the computing time of PSAP and ES both increases, however, the computing time using PSAP is reduced by 22\% compared with ES. ES tries all the possible insertion positions while PSAP only tries a part of them since the requests violating QoS constraints are excluded. From (\ref{Eqn:psi}) and (\ref{Eqn:psi2}), we know that, if the urban area increases, more computation will be reduced using PSAP, i.e., PSAP will have better performance in traffic big data scenarios.

The \emph{second} metric is RRCC, which is displayed by Fig.~\ref{Fig:ReductionRatioComputation} where the number of PVs is 70, and the number of requests is set to 700, 800, and 900 respectively during 11:00$\sim$11:59. We see that, the RRCC in case $\mathbb{A}$ is much higher than that in case $\mathbb{B}$, and the RRCC in case $\mathbb{C}$ is between the other two cases. The results can be inferred from (\ref{Eqn:psi}) and (\ref{Eqn:psi2}). For example, when the number of requests is 900, the RRCC in three cases $\mathbb{A}$, $\mathbb{B}$ and $\mathbb{C}$ is 40.2\%, 15.1\%, and 31.8\% respectively. The result can be explained as follows. In case $\mathbb{A}$, both the origin and destination of a new request should be in PSA of the PV, while in case $\mathbb{B}$, only the origin should be in PSA. Therefore, the number of requests which satisfy case $\mathbb{A}$ is much smaller than that of case $\mathbb{B}$. In case $\mathbb{C}$, all the points of the current path of one PV should fall in PSA, which is related to the number of requests which are being served and waiting to be served, therefore, it is hard to predict the result in case $\mathbb{C}$ under such particular constraints.

\begin{figure*}[htbp]
  \begin{minipage}[b]{0.33\linewidth}
    \centering
    \includegraphics[height=0.80\linewidth,width=0.99\linewidth]{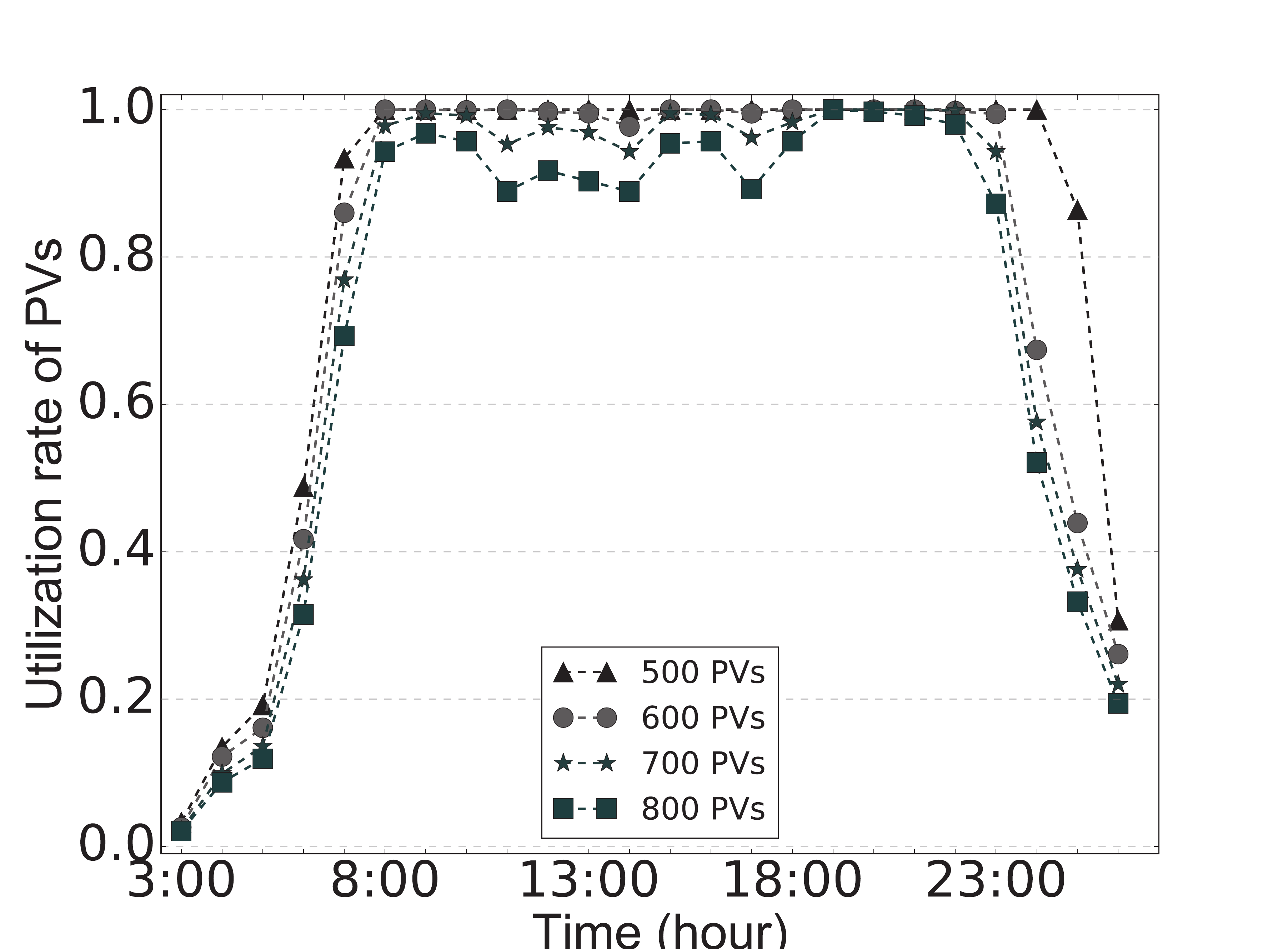}
    \caption{Utilization rate.}
    \label{Fig:UtilizationRate}
  \end{minipage}
  \begin{minipage}[b]{0.33\linewidth}
  \centering
  \includegraphics[height=0.80\linewidth,width=0.99\linewidth]{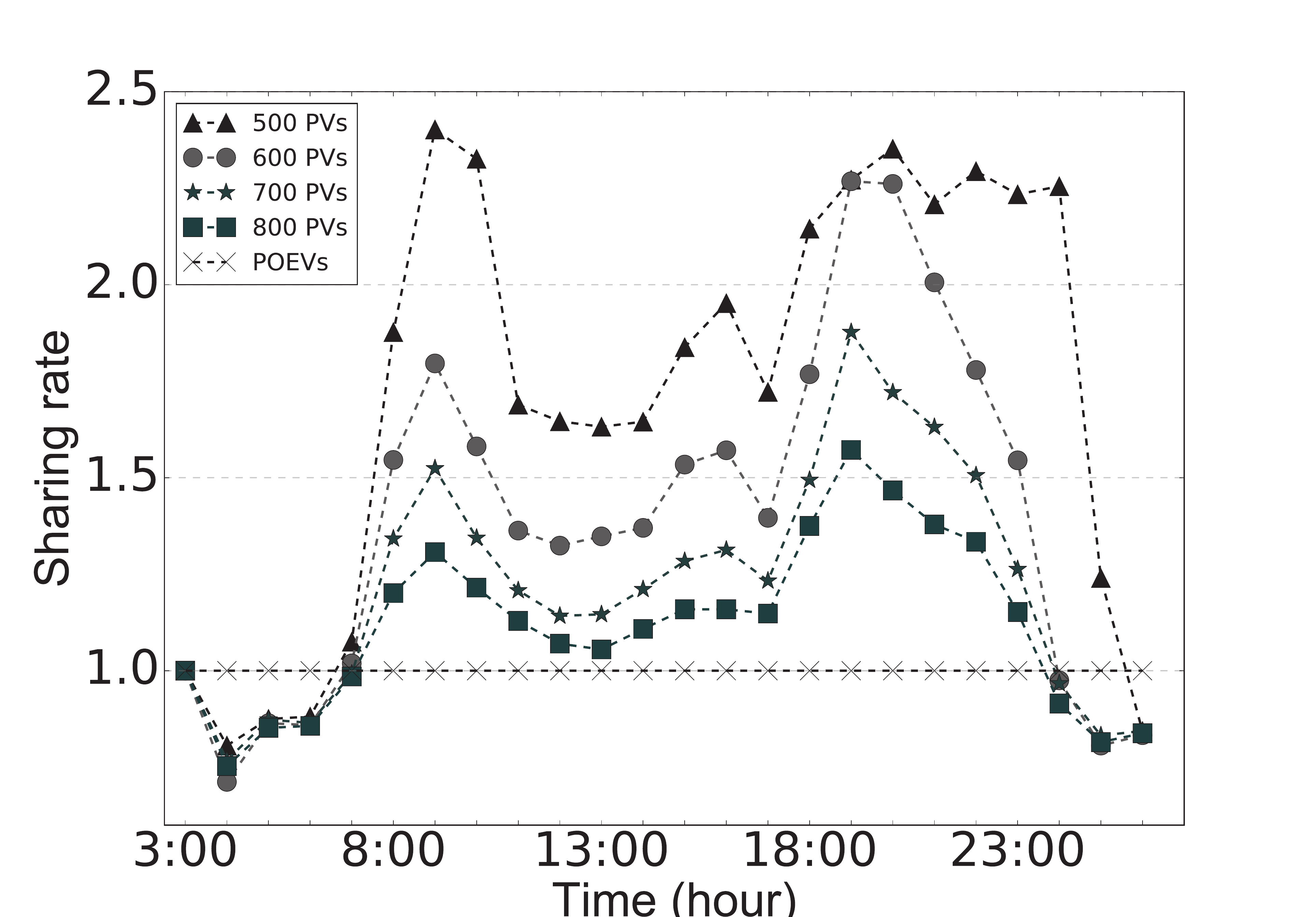}
  \caption{Sharing rate.}
  \label{Fig:SharingRate}
  \end{minipage}
  \begin{minipage}[b]{0.33\linewidth}
    \centering
    \includegraphics[height=0.80\linewidth,width=0.99\linewidth]{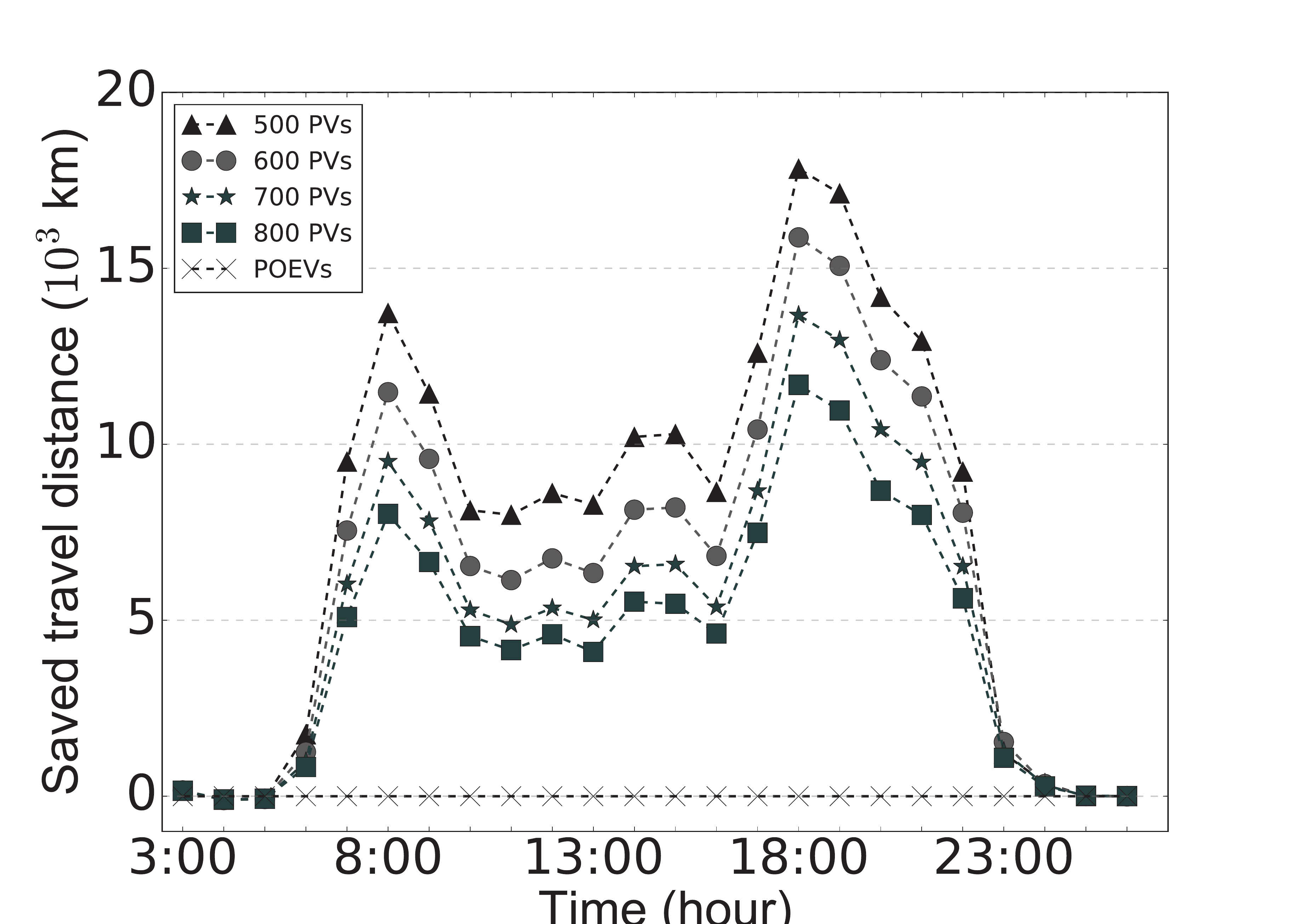}
    \caption{Saved travel distance.}
    \label{Fig:SavedTravelDistance}
  \end{minipage}
\end{figure*}

Next, we put 500$\sim$800 PVs to the system to serve 75,014 requests of the day. The average waiting-travel time of requests is (18.7, 14.1), (3.2, 13.7), (2.5, 13.1) and (2.1, 12.7) minutes using 500, 600, 700, and 800 PVs respectively. This means that if more PVs are put to the system, the waiting and travel time will be less, since the service quality level will improve with less passengers sharing one vehicle.

The \emph{third} metric is the total travel distance, which is plotted by Fig.~\ref{Fig:TravelDistanceComparison}. We see that, when we put more than 500 PVs to serve passengers, the total travel distance of PVs increases slowly since passengers share common paths during their trips. The travel distance of PVs is reduced by 24\%$\sim$39\% compared with that by POEVs when the number of PVs varies from 500 to 800.

The \emph{fourth} metric is the utilization rate, which is depicted by Fig.~\ref{Fig:UtilizationRate}. We can see that, when we put 500 PVs to the system, during 8:00$\sim$24:00, almost all PVs are busy on serving passengers since the number of PVs is too small to serve such a large number of requests. While if we put more PVs, e.g., 800, the utilization rate of PVs almost increases to one during peak time such as 8:00, yet decreases obviously during non-peak time such as 12:00.

The \emph{fifth} metric is the sharing rate, which is presented by Fig.~\ref{Fig:SharingRate}. We see that, the less PVs do we put, the larger the sharing rate will be. At non-peak time such as 4:00, the sharing rate is less than one, since very few passengers share PVs, however, PVs have to travel some distance to pick up them. However, at peak time such as 19:00, the sharing rate improves obviously since more passengers share PVs.

The \emph{sixth} metric is the saved travel distance, which is shown by Fig.~\ref{Fig:SavedTravelDistance}. We can see that, at non-peak time, the saved travel distance is very small, even is negative at 4:00, since very few passengers share PVs at this time. However, at peak time, the saved travel distance increases largely especially at 8:00 and 18:00 since the more passengers join ride-sharing, the more the saved travel distance will be.

Now, we discuss the energy efficiency of PV systems. We assume the travel distance of PVs and POEVs is positively correlated with the energy consumption. From Fig.~\ref{Fig:TravelDistanceComparison}, we can infer that the average energy cost will be dropped by 24\%$\sim$39\%. In PV systems, although passengers may sacrifice some comfort during their trips, the discomfort emerged from ride-sharing is limited since PASP tries to balance the waiting time of passengers and restrict the detour of each passenger to a tolerated scope. Therefore, the QoS of passengers is guaranteed, and the low-cost ride-sharing will attract more passengers to choose PV systems.

Here, we discuss the particular features of PV systems in peak and non-peak time. According to Fig.~\ref{Fig:TravelDistanceComparison} and Fig.~\ref{Fig:SavedTravelDistance} which show the total travel distance and the saved travel distance, respectively, we can infer that, if too many PVs join the system, the profits will be reduced for the sake of parking fees, however, too few PVs can save more travel distance with the cost of degrading the service quality level of passengers. At traffic peak time, the saved travel distance is much more than that in non-peak time, since more passengers share PVs. Therefore, the number of PVs put to the system is an important factor which balances the PV system profits and passenger QoS. We conclude that PV systems are more practical and profitable in traffic peak time of crowd urban areas, however, in traffic non-peak time or suburban areas with few passengers, the performance of PV systems is as good as that in the former scenario.

\section{Discussion and Conclusion} \label{Sec:Conclusion}


To deal with the online/dynamic ride-sharing path planning problem for PV systems, we propose a solution
based on a limited potential search area for each vehicle to filter out the requests that violate passenger QoS constraints such as detour, therefore, the global search is reduced to local search and the computational complexity is reduced. It also considers the comfort of passengers (e.g., waiting time and detour) and the total travel distance of PVs. Therefore, passengers can enjoy their peer-to-peer ride-sharing services with sacrificing a little ride comfort. We also analyze the reduction ratio of computational complexity using the proposed solution, which is related to the current paths and service lists of vehicles, and the locations of origin-destination pairs of unscheduled requests. Through our simulations based on Manhattan taxi data sets, the computation time of the proposed approach is greatly reduced than the exhaustive search method, and the average energy cost for each passenger by PVs is largely reduced compared with that by private owned electric vehicles.

Here, we consider less about the passenger preferences, and in the future we will focus on the ride-sharing path planning methods with more preferences, e.g., the maximum number of shared persons, the pick-up/drop-off point selection, and the point of interest selection. In addition, we will design the ride-sharing path planning solutions aiming at solving the common last mile problem from homes or work places to subway stations or bus stops.


\begin{IEEEbiography}[{\includegraphics[width=1in,height=1.21in,clip,keepaspectratio]{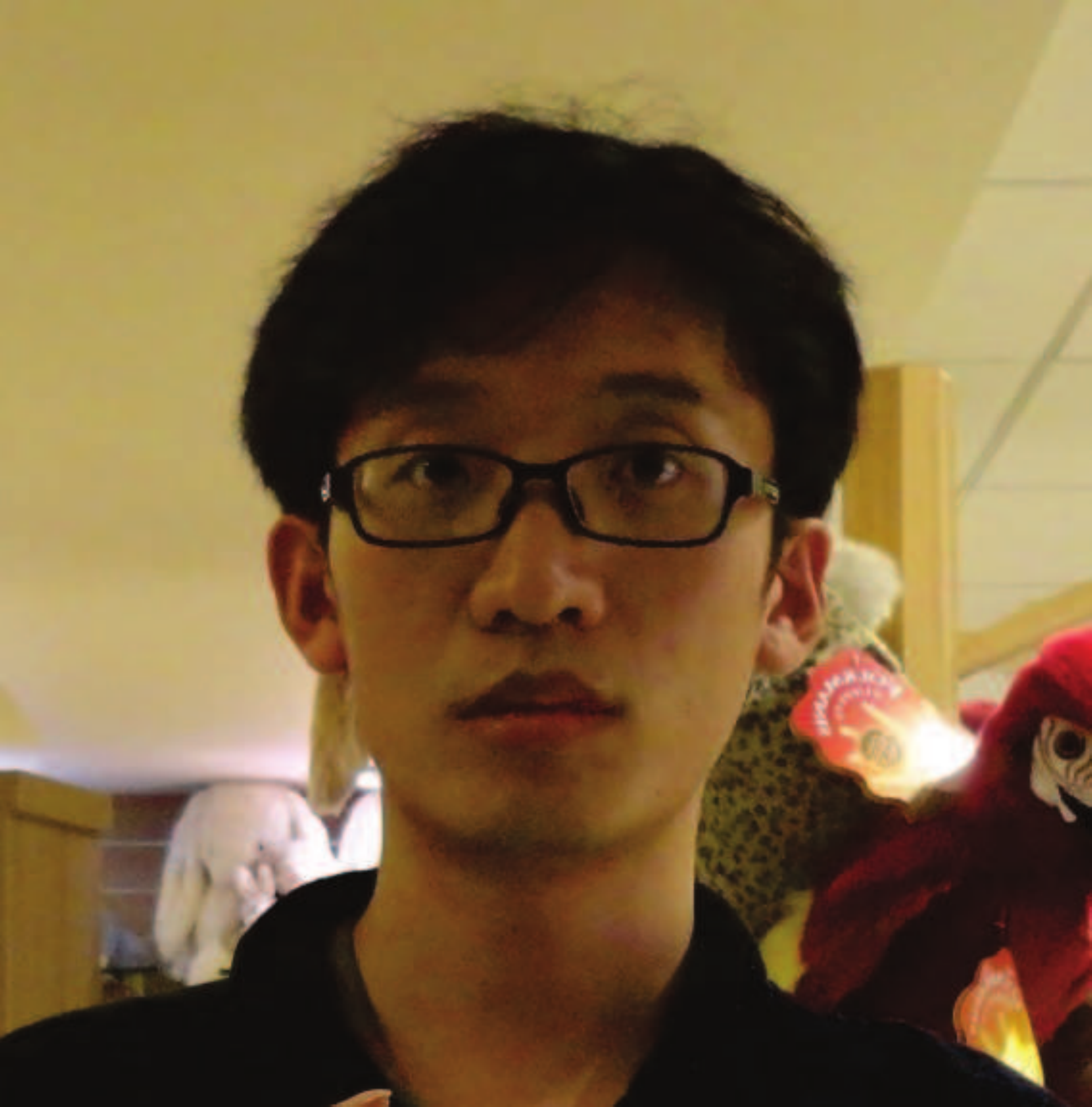}}]\\~~{\textbf{Ming~Zhu}} is now a postdoctoral researcher and assistant researcher in Shenzhen Institutes of Advanced Technology, Chinese Academy of Sciences, Shenzhen, China. He received the PhD degree in Computer Science and Engineering in Shanghai Jiao Tong University, Shanghai, China. A part of this work is finished in Shanghai Jiao Tong University.

His research interests are in the area of big data, artificial intelligence, internet of things, and wireless communications.
\end{IEEEbiography}

\begin{IEEEbiography}[{\includegraphics[width=1.40in,height=1.07in,clip,keepaspectratio]{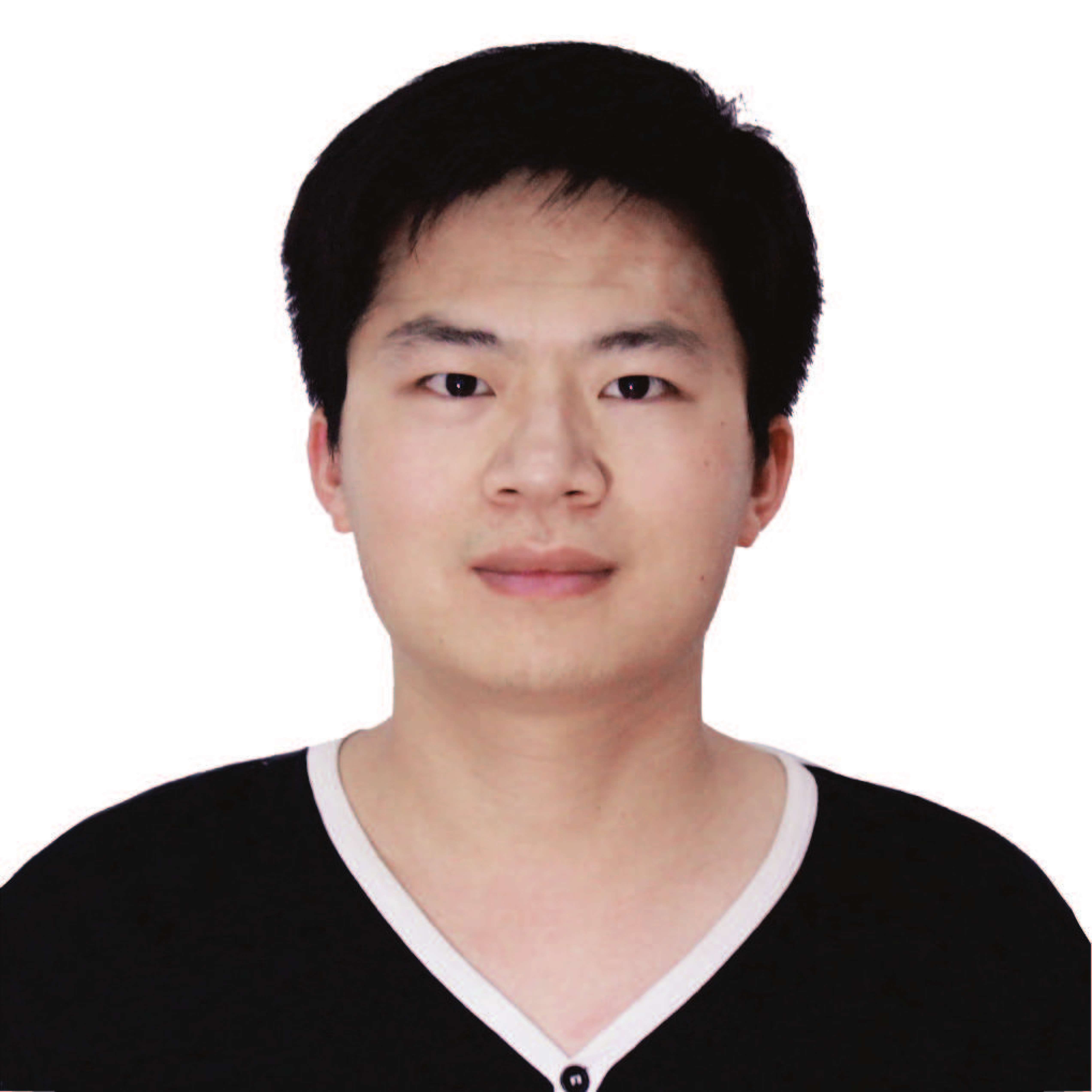}}]{Xiao-Yang~Liu} received his B.Eng. degree in computer science from Huazhong University of Science and Technology, China, in 2010. He is currently a joint PhD in the Department of Electrical Engineering, Columbia University, and in the Department of Computer Science and Engineer, Shanghai Jiao Tong University.

His research interests include tensor theory, deep learning, nonconvex optimization, big data analysis and homomorphic encryption, cyber-security and wireless communication.
\end{IEEEbiography}

\begin{IEEEbiography}[{\includegraphics[width=1in,height=1.21in,clip,keepaspectratio]{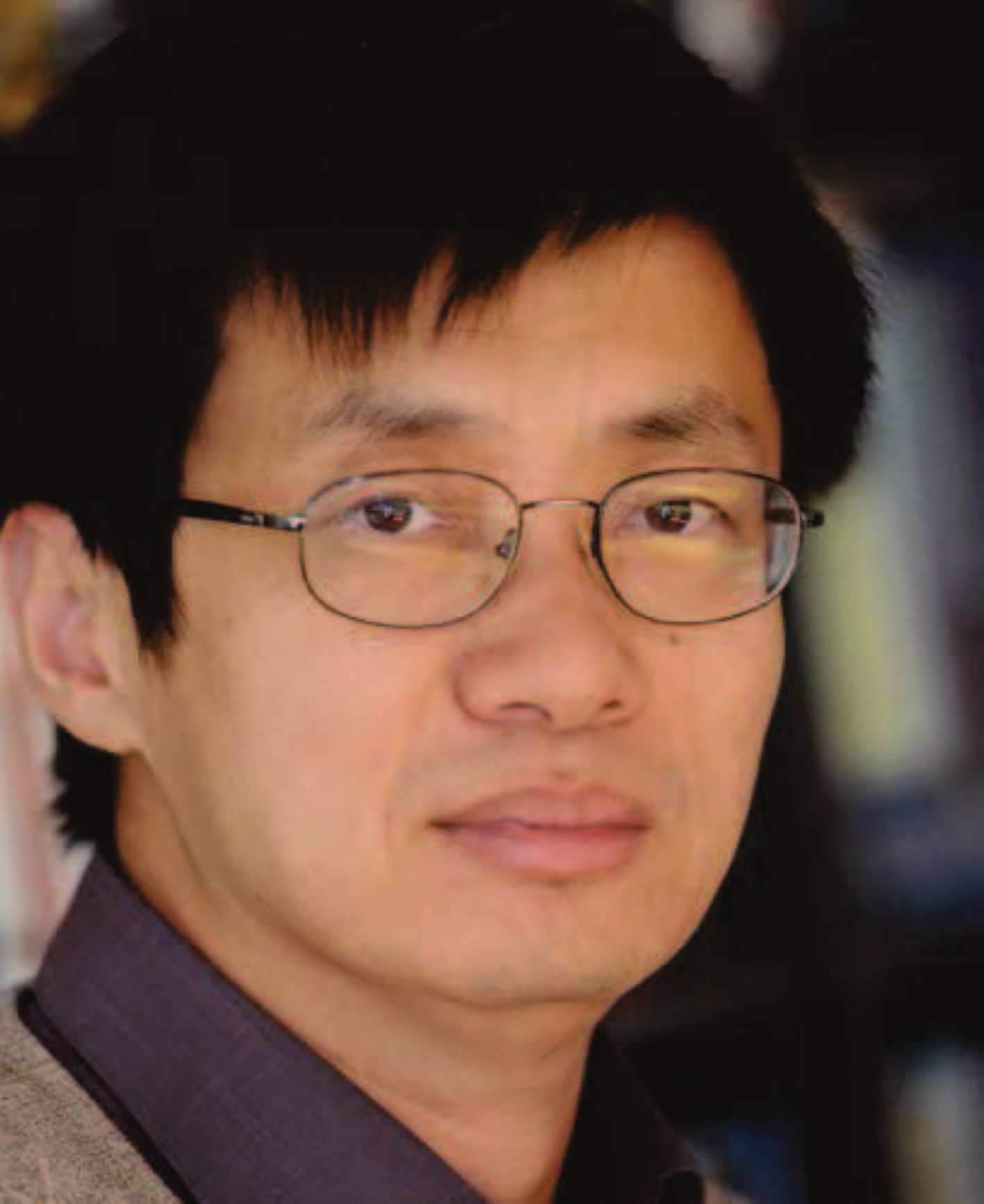}}]{Xiaodong~Wang} (S'98-M'98-SM'04-F'08) received the Ph.D. degree in electrical engineering from Princeton University. He is currently a Professor of electrical engineering with Columbia University, New York NY, USA. His research interests fall in the general areas of computing, signal processing, and communications. He has authored extensively in these areas. He has authored the book entitled Wireless Communication Systems: Advanced Techniques for Signal Reception, (Prentice Hall, 2003). His current research interests include wireless communications, statistical signal processing, and genomic signal processing. He has served as an Associate Editor of the IEEE TRANSACTIONS ON COMMUNICATIONS, the IEEE TRANSACTIONS ON WIRELESS COMMUNICATIONS, the IEEE TRANSACTIONS ON SIGNAL PROCESSING, and the IEEE TRANSACTIONS ON INFORMATION THEORY. He is an ISI Highly Cited Author. He received the 1999 NSF CAREER Award, the 2001 IEEE Communications Society and Information Theory Society Joint Paper Award, and the 2011 IEEE Communication Society Award for Outstanding Paper on New Communication Topics.
\end{IEEEbiography}

\end{document}